\documentclass{article}

\PassOptionsToPackage{numbers, compress}{natbib}

\usepackage[preprint]{neurips_2025}

\usepackage[utf8]{inputenc} 
\usepackage[T1]{fontenc}    
\usepackage{hyperref}       
\usepackage{url}            
\usepackage{booktabs}       
\usepackage{amsfonts}       
\usepackage{nicefrac}       
\usepackage{microtype}      
\usepackage{xcolor}         
\usepackage{amsmath}
\usepackage{multirow}
\usepackage{makecell}
\usepackage{diagbox}

\usepackage{graphicx}
\usepackage{wrapfig}

\usepackage{listings}
\usepackage{xcolor}
\usepackage{caption}

\definecolor{codegray}{rgb}{0.95,0.95,0.95}
\definecolor{commentgreen}{rgb}{0.0,0.5,0.0}
\definecolor{keywordblue}{rgb}{0.0,0.0,0.6}
\definecolor{stringred}{rgb}{0.58,0.0,0.0}

\lstdefinestyle{pythonpseudo}{
    backgroundcolor=\color{codegray},
    basicstyle=\ttfamily\footnotesize,
    keywordstyle=\color{keywordblue}\bfseries,
    stringstyle=\color{stringred},
    commentstyle=\color{commentgreen}\itshape,
    numbers=left,
    numberstyle=\tiny\color{gray},
    breaklines=true,
    tabsize=2,
    showstringspaces=false,
    frame=none
}

\usepackage{multirow}
\usepackage{graphicx} 

\usepackage{booktabs}
\usepackage{rotating}
\usepackage{array}
\usepackage{diagbox}
\usepackage{subcaption} 

\usepackage{amsmath, amsthm}

\newtheorem{proposition}{Proposition}
\newtheorem{lemma}{Lemma}
\newtheorem{definition}{Definition}
\newtheorem{corollary}{Corollary}

\usepackage{algorithm}
\usepackage{algpseudocode}

\title{Beyond Accuracy: EcoL2 Metric for \\ Sustainable Neural PDE Solvers}

\author{
  Taniya Kapoor\thanks{These authors contributed equally.} \\
  Department of Engineering Structures\\
  TU Delft, The Netherlands \\
  \texttt{t.kapoor@tudelft.nl}
  \And
  Abhishek Chandra\footnotemark[1] \\
  Department of Electrical Engineering \\
  TU Eindhoven, The Netherlands
  \AND
  Anastasios Stamou \\
  Laboratory of Earthquake Engineering \\
  National Technical University of Athens, Greece
  \And
  Stephen J Roberts \\
  Machine Learning Research Group \\
  University of Oxford, UK 
}

\begin{document}

\maketitle

\begin{abstract}
Real-world systems, from aerospace to railway engineering, are modeled with partial differential equations (PDEs) describing the physics of the system. Estimating robust solutions for such problems is essential. Deep learning-based architectures, such as neural PDE solvers, have recently gained traction as a reliable solution method. The current state of development of these approaches, however, primarily focuses on improving accuracy. The environmental impact of excessive computation, leading to increased carbon emissions, has largely been overlooked. This paper introduces a carbon emission measure for a range of PDE solvers. Our proposed metric, EcoL2, balances model accuracy with emissions across data collection, model training, and deployment. Experiments across both physics-informed machine learning and operator learning architectures demonstrate that the proposed metric presents a holistic assessment of model performance and emission cost. As such solvers grow in scale and deployment, EcoL2 represents a step toward building performant scientific machine learning systems with lower long-term environmental impact.
\end{abstract}

\section{Introduction}
\label{sec:introduction}
Partial differential equations (PDEs) are integral to computational sciences and engineering. PDEs have broad applicability, ranging from ecology \cite{krasnow2025making} to complex engineering systems such as aircraft \cite{lye2021iterative} and railways \cite{kapoor2024neural}. 
Simulation of PDEs has a long history, with numerous numerical methods developed over the past century, such as finite element \cite{quarteroni2008numerical} and finite volume methods \cite{leveque2002finite}. For decades, such methods have been considered optimal and underpin numerous real-world applications.

Recently, with advances in artificial intelligence (AI), the focus of PDE simulation has shifted toward deep learning-based approaches \cite{huang2025partial}. This shift has led to a new class of methods, the so-called neural PDE solvers. However, as PDE solvers evolve into more and more complex AI-based methods, the criteria used to evaluate their value remain essentially unchanged. With the advent of AI-based methods, however, the environmental impact of PDE solvers, has significantly escalated \cite{mcgreivy2024weak}. Neural PDE solvers adopt several benefits from core AI \cite{karniadakis2021physics, azizzadenesheli2024neural} but also inherit high computational demands \cite{deoperator, muller2023achieving, datar2024solving}, increasing carbon emissions. This increased carbon footprint of neural PDE solvers is overlooked and raises concerns about their environmental sustainability, presenting a paradox: methods originally developed to serve societal needs now pose unintended challenges to the society they aim to benefit.

Hitherto, despite rapid progress, neural PDE solvers\textemdash like most AI models \cite{thomas2022reliance}\textemdash are still evaluated primarily on accuracy, whether through error plots or numerical relative errors. This myopic view reflects Goodhart's law \cite{goodhart1975problems}: "\emph{When a measure becomes the target, it ceases to be a good measure}." Figure~\ref{fig1} illustrates this issue by presenting the performance of a range of neural PDE solvers on canonical PDEs: physics-informed machine learning methods—namely, physics-informed neural networks (PINNs) \cite{raissi2019physics}, PINNsFormer (PF) \cite{zhaopinnsformer}, and separable PINN (SPINN) \cite{cho2023separable} are evaluated for the advection equation; and neural operator approaches—namely, deep operator network (DON) \cite{lu2021learning}, Fourier neural operator (FNO) \cite{lifourier}, and convolutional neural operator (CNO) \cite{raonic2023convolutional} are evaluated over KdV equation. Details of the experiment are presented later in the paper (Section~\ref{Sec:numerical_experiments}). The Left subfigure compares the methods based on relative L2 error, where the models seem comparable. However, the Middle subfigure showcases the models' associated carbon emissions, indicating variability in the carbon footprint of these models\textemdash yet this dimension is never considered in comparative evaluations. 

\begin{figure}[t]
    \centering
    \vspace{-20pt}
    \includegraphics[width=0.60\textwidth]{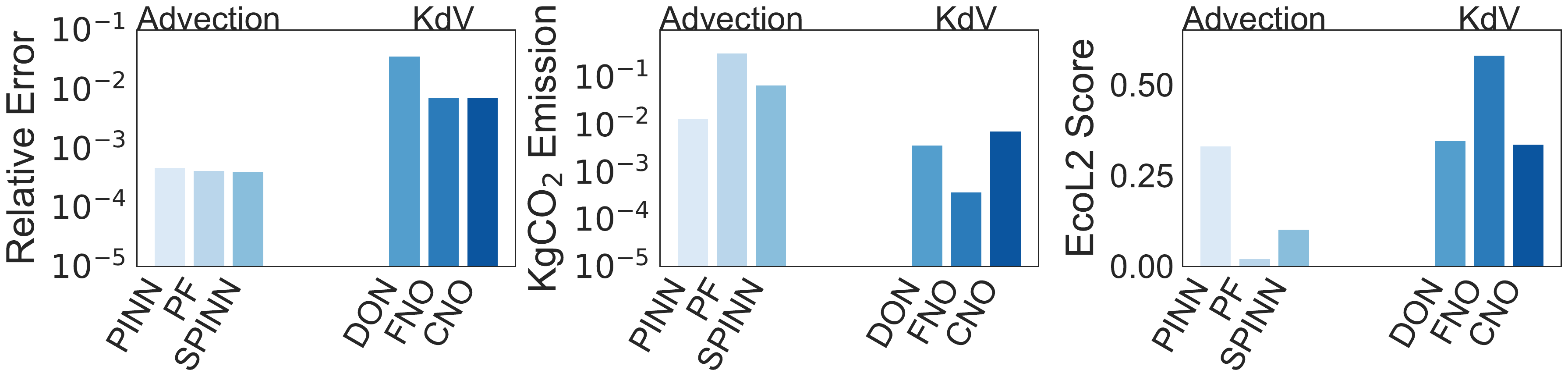}
    \caption{Performance comparison of physics-informed learning methods and neural operators on the advection and KdV equations, respectively. Models are evaluated on relative error \textbf{(Left)}, the corresponding carbon emissions (kgCO$_2$) \textbf{(Middle)}, and the proposed EcoL2 metric \textbf{(Right)}. Comparison of models solely based on relative error provides a myopic view as they have varying carbon footprints. The proposed EcoL2 metric (higher values are preferable) captures this trade-off, offering a performant perspective of solver performance.}
    \label{fig1}
    \vspace{-15pt}
\end{figure}

Figure~\ref{fig1} highlights that focusing solely on accuracy creates a bias in developing these models and is not long-term sustainable, necessitating the need for performant evaluation frameworks that consider both acccuracy and environmental sustainability. In response, this paper proposes measuring the carbon footprint of neural PDE solvers to promote sustainable growth. Furthermore, the paper introduces the EcoL2 metric, which explicitly balances relative error with carbon emissions and takes a step toward performant evaluation of scientific machine learning methods. As shown in Figure~\ref{fig1} (Right), EcoL2 combines these two aspects into a single interpretable score, presenting a holistic assessment of model performance (where higher EcoL2 values are preferable).

This paper introduces four distinct types of carbon emissions associated with the lifecycle of neural PDE solvers. Through empirical evaluations of canonical PDEs across several neural PDE solvers of the PINN family and neural operators, this paper highlights the importance of the proposed EcoL2 metric. Moreover, the proposed metric is adaptable, \textit{i.e.,} depending on the application, accuracy, or carbon efficiency can be prioritized by tuning its weighting hyperparameters. Furthermore, the paper demonstrates that hardware choice plays an important role in the carbon footprint of neural PDE solvers. The paper also studies the correlation of carbon emissions, if any, with computational cost and different geographical regions, emphasizing the broader deployment context. Finally, the applicability of the proposed metric is illustrated for various ansatz and tasks, including function approximation and symbolic model discovery.

To summarize, the contributions of this paper are as follows: 1) This paper introduces four distinct sources of carbon in neural PDE solvers: embodied carbon (data acquisition), developmental carbon (hyperparameter tuning), operational carbon (training), and inference carbon (deployment), together defining the total carbon footprint; 2) This paper proposes a new evaluation metric, EcoL2, which combines solver accuracy with its CO$_2$-equivalent emissions to foster environmentally informed model assessment; 3) The EcoL2 metric is adaptable, where accuracy or environmental efficiency may be prioritized depending on specific application requirements; 4) The utility of EcoL2 is demonstrated through empirical analysis of two solver families, PINNs and neural operators, on several canonical PDE benchmarks. The metric is further shown to be broadly applicable beyond PDE solving.

The rest of the paper is structured as follows. Section~\ref{Sec:life_cycle} introduces the carbon lifecycle in neural PDE solvers. Section~\ref{Sec:proposed_metric} presents the proposed EcoL2 metric and analyzes it mathematically. Section~\ref{Sec:numerical_experiments} presents the numerical experiments to show the performance of EcoL2. Section~\ref{Sec:related} details the related work. Discussions and the main conclusions drawn from this study are collated in Section~\ref{Sec:conclusions}.

\section{Lifecycle carbon assessment of neural PDE solvers}
\label{Sec:life_cycle}
This section presents the lifecycle carbon assessment for neural PDE solvers, which forms the foundation of the proposed EcoL2 metric. Figure~\ref{fig2} presents neural PDE solvers' lifecycle and carbon footprint. The proposed assessment categorizes carbon emissions into four primary sources, each corresponding to a distinct stage in developing and deploying neural PDE models.
\begin{wrapfigure}{r}{0.6\textwidth}
    \centering
    \includegraphics[width=0.60\textwidth]{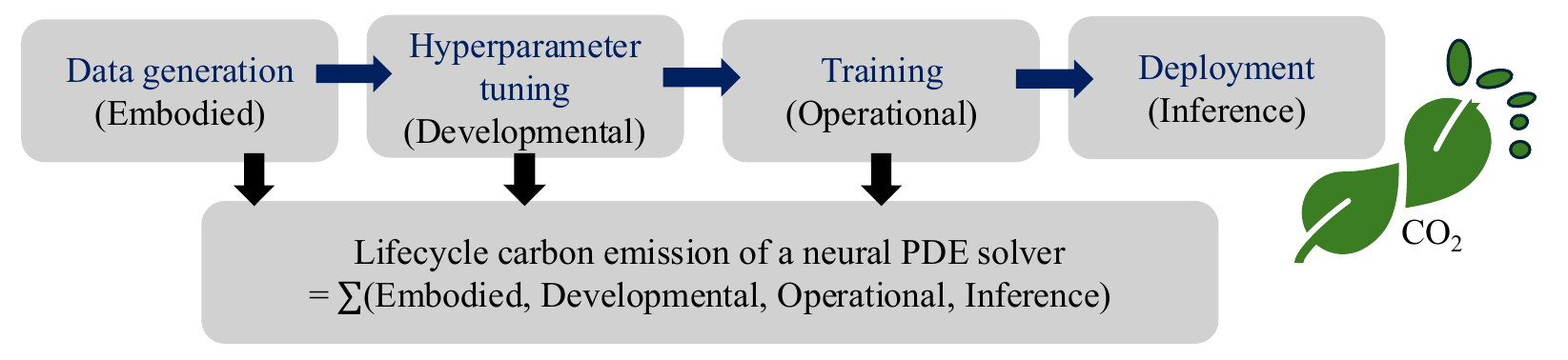}
    \caption{Lifecycle carbon of neural PDE solvers}
    \label{fig2}
\end{wrapfigure}
The first source is the data required to train these solvers. The carbon associated with this stage is referred to as \textbf{embodied carbon} (\(C_e\)) and accounts for emissions associated with data generation. For forward problems, PINN often requires no training data beyond the well-posed PDE and does not contribute to embodied carbon. In contrast, neural operator models rely on large, high-fidelity datasets generated through numerical solvers \cite{takamoto2022pdebench, ohana2024well}. These numerical solver runs embed a significant carbon contribution into the pipeline before any neural training begins. This terminology mirrors the concept of embodied carbon in structural engineering, where emissions due to material extraction and development are accounted for even before infrastructure construction \cite{ghorbany2025automating}.

The second source is the computations required to tune the model's hyperparameters leading to \textbf{developmental carbon} (\(C_d\)). Finding the right combination of hyperparameters for neural PDE solvers is computationally intensive because of the large number of hyperparameters involved in the state-of-the-art models. PINN-based models and neural operators can require hundreds of trial runs to identify optimal configurations, leading to substantial emissions during this stage \cite{mishra2023estimates, raonic2023convolutional}.

The third contributor is the \textbf{operational carbon} (\(C_o\)), which includes emissions from final training runs. Once the model architecture and hyperparameters are fixed, solvers are trained for extended epochs to refine performance or explore additional observables. These final training runs can be significant, particularly for large models \cite{herde2024poseidon}.

Finally, \textbf{inference carbon} (\(C_i\)) refers to emissions from model evaluation or deployment. Although a single forward pass costs minimal, repeated usage can lead to a measurable cumulative impact. Hence, the total lifecycle carbon emission (\(C\)) of a neural PDE solver is calculated as:
\begin{equation}
C = C_e + C_d + C_o + C_i.
\label{eq:carbon}
\end{equation}
This breakdown highlights the importance of assessing and reporting carbon emissions at each stage. This paper argues that carbon accounting and reporting should become standard when developing neural PDE solvers to foster sustainable scientific machine learning methods.

\section{Proposed metric: EcoL2}
\label{Sec:proposed_metric}
The proposed EcoL2 metric is defined as follows:
\begin{equation}
\textrm{EcoL2} =
\frac{1 - e^{\log_{\alpha}(\mathcal{R})}}{1 + \beta (C_e + C_d + C_0 + C_i \cdot n_{\text{infer}})},
\label{eq:ecol2}
\end{equation}
here, $\alpha$, $\beta$ are hyperparameters, and $\mathcal{R}$ denotes the relative L2 error of the model. $C_e$ is embodied carbon, $C_d$ is developmental carbon, $C_0$ is operational carbon, $C_{i}$ is inference carbon, $n_{\mathrm{infer}}$ is the number of inferences. The term $C_e + C_d + C_0 + C_i \cdot n_{\text{infer}}$ represents the carbon footprint, accounting for emissions from training stages as well as after $n_{\mathrm{infer}}$ number of inferences.
\begin{wrapfigure}{r}{0.4\textwidth}
    \centering
    \includegraphics[width=0.40\textwidth]{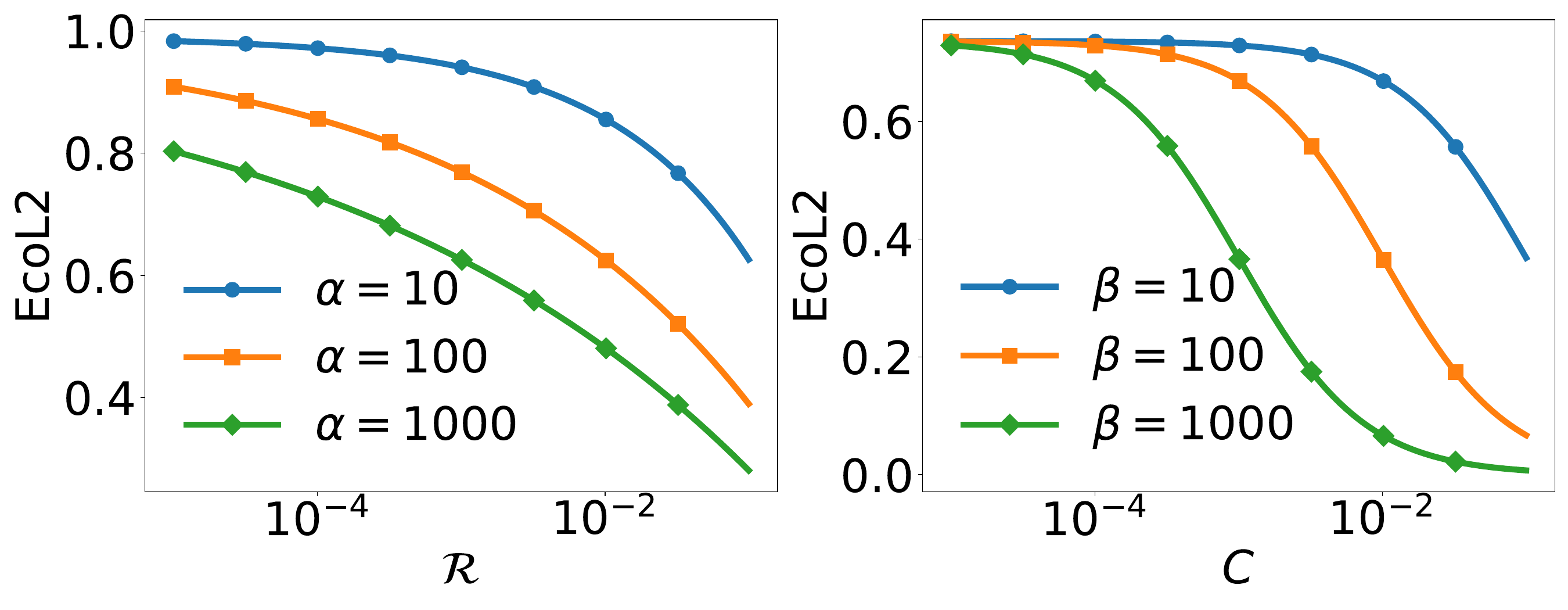}
    \caption{EcoL2 for varying $\alpha$, $\beta$ values}
    \label{fig3}
\end{wrapfigure}
\textbf{Motivation: }The formulation of EcoL2 integrates accuracy and environmental cost into a single score. EcoL2's numerator (hereafter referred to as numerator) uses an exponential-log transformation of the relative L2 error, $e^{\log_\alpha \mathcal{R}}$, and subtracts it from unity to capture the nonlinear value of accuracy gains. Modulating by $\alpha$ emphasizes that EcoL2 score improvements in low-error regimes are harder to achieve (where it matters most) than equivalent gains at higher error levels (irrespective of the value of $\alpha$), as shown in Figure~\ref{fig3} (Left), where $\beta=100$, $C=1e-4$. EcoL2's denominator (hereafter referred to as denominator) aggregates and scales the model's carbon footprint by $\beta$. This linear composition of $\beta$ with carbon is based on the assumption of the proportional nature of energy usage and controls the metric's sensitivity to environmental cost. Figure~\ref{fig3} (Right) shows the variation of EcoL2 for different $\beta$ values when $\alpha=100$, $\mathcal{R}=1e-4$. 

\textbf{Adaptability: }The hyperparameters $\alpha$ and $\beta$ flexibly prioritize either model accuracy or sustainability based on application needs. A higher $\alpha$ increases the influence of accuracy, making the metric more suitable for domains where even minor errors are critical, such as high-stakes decision systems. In contrast, increasing $\beta$ accounts for carbon costs more heavily, which is important in large-scale deployments. See Figure~\ref{fig3} for the expected performance of EcoL2 under varying values of $\alpha$ and $\beta$. 

Next, EcoL2 is analyzed mathematically based on the following definition and assumptions.
\begin{definition}
A neural PDE solver is considered inaccurate in this work if $\mathcal{R} \geq 0.1$, which corresponds to 10\% or more relative error compared to the reference solution.
\end{definition}
\textbf{Assumptions: }EcoL2 metric assumes that the PDE solver is not inaccurate, hence $\mathcal{R}\in (0,0.1)$. In addition, for EcoL2 to be well-defined, $\alpha > 0$ and $\alpha \neq 1$, considering the domain of the logarithmic function. However, since $\alpha$ governs the weighting of accuracy, and PDE solving typically requires high accuracy, this paper considers the domain to be $\alpha \in [10, 1000]$. Similarly, $\beta \geq 1$ is assumed to scale the carbon cost. The total carbon footprint $C > 0$, though individual components may be zero in special cases (discussed later). 
\begin{proposition}
\label{prop}
The \emph{EcoL2} score is bounded within the interval $(0,\, 1)$, i.e., \emph{EcoL2} $\in (0,\, 1)$ , for all $\mathcal{R} \in (0,\, 0.1)$, $\alpha \in [10, 1000]$, $\beta \geq 1$, and $C > 0$.
\end{proposition}
\begin{proof}
The proposition is proved using logarithmic, exponential rules, and positivity and boundedness properties. Using the change-of-base identity, the numerator could be rewritten as:
\[
1 - e^{\log_{\alpha}(\mathcal{R})}
= 1 - e^{\frac{\ln \mathcal{R}}{\ln \alpha}}
= 1 - \left( e^{\ln \mathcal{R}} \right)^{1/\ln \alpha}
= 1 - \mathcal{R}^{1/\ln \alpha}
\]
Since $\alpha \geq 10 > 1$, it follows that $\ln \alpha > 0$ and hence $1 / \ln \alpha > 0$. Also, $\mathcal{R} \in (0,\, 0.1)$, therefore, $\mathcal{R}^{1 / \ln \alpha} \in (0, 1)$. This results in the numerator $1 - \mathcal{R}^{1 / \ln \alpha} \in (0, 1)$.

The denominator is $1 + \beta (C_e + C_d + C_0 + C_i \cdot n_{\text{infer}})$, where $\beta \geq 1$, and also $C>0$ making the denominator strictly greater than 1. Thus, EcoL2 is the ratio of a value in $(0, 1)$ over a value greater than 1, implying $\textrm{EcoL2} \in (0, 1)$.
\end{proof}

\begin{lemma}
\label{lemma1}
$\mathrm{EcoL2} \to 1 \enspace \iff \enspace \mathcal{R} \to 0 \enspace \& \enspace C \to 0$.
\end{lemma}

\begin{proof}
The Lemma is proved for both, if ($\Rightarrow$) and only if ($\Leftarrow$) parts. Starting with, ($\Rightarrow$) if $\mathcal{R} \to 0$ and $C \to 0$, then it needs to be proved that $\mathrm{EcoL2} \to 1$.

As $\mathcal{R} \to 0$, and $\ln \alpha > 0$, hence, $\log_{\alpha}(\mathcal{R}) = \frac{\ln \mathcal{R}}{\ln \alpha} \to -\infty$. Therefore, $e^{\log_{\alpha}(\mathcal{R})} \to 0 \Rightarrow 1 - e^{\log_{\alpha}(\mathcal{R})} \to 1$.

Next, $C \to 0$ and all carbon components are nonnegative, implying that $C_e, C_d, C_o, C_i \to 0$. Hence, $\beta (C_e + C_d + C_o + C_i \cdot n_{\text{infer}}) \to 0$,
which implies $1 + \beta (C_e + C_d + C_o + C_i \cdot n_{\text{infer}}) \to 1$.

Therefore, both the numerator and the denominator tend to 1, so $\mathrm{EcoL2} \to 1$, this proves the if part.

($\Leftarrow$) Next it is proved that if $\mathrm{EcoL2} \to 1$, then $\mathcal{R} \to 0$ $\&$ $C \to 0$. From the proof of Proposition~\ref{prop}, the numerator is bounded in the interval $(0, 1)$, the denominator is strictly greater than 1, and $\mathrm{EcoL2} < 1$. Hence, for $\mathrm{EcoL2} \to 1$, this can only occur if the numerator and the denominator $\to 1$. For numerator, as $\alpha$ is bounded, $1 - \mathcal{R}^{1/\ln \alpha} \to 1 \Leftrightarrow \mathcal{R}^{1/\ln \alpha} \to 0 \Leftrightarrow \mathcal{R} \to 0$.

For the denominator $\to 1$, it must be that:
$\beta (C_e + C_d + C_o + C_i \cdot n_{\text{infer}}) \to 0$, which, since $\beta \geq 1$ and all carbon components are non negative, implies: $C_e, C_d, C_o, C_i \to 0 \Rightarrow C \to 0$.
\end{proof}
\begin{lemma}
\label{lemma2}
$\mathrm{EcoL2} \to 0 \enspace$ if either $\mathcal{R} \to 1$ or $C \to \infty$.
\end{lemma}
\begin{proof}
Both possible cases are considered separately.
\textit{Case 1:} The case $\mathcal{R} \to 1$ refers to inaccurate PDE solvers and EcoL2 need not be computed. The proof is performed for completeness, If $\mathcal{R} \to 1$, then $\log_{\alpha}(\mathcal{R}) \to 0$, and hence $e^{\log_{\alpha}(\mathcal{R})} \to 1$, so the numerator $1 - e^{\log_{\alpha}(\mathcal{R})} \to 0$. Since the denominator is strictly greater than 1, $\mathrm{EcoL2} \to 0$.
\textit{Case 2:} If $C \to \infty$, then it is only possible when at least one of the carbon components $\to \infty$, and thus the denominator $\to \infty$. Meanwhile, the numerator is bounded between 0 and 1, so $\mathrm{EcoL2} \to 0$. This completes the proof.
\end{proof}
\begin{corollary}
\label{cor}
A value of \emph{EcoL2} close to \( 0 \) indicates poor performance of the neural PDE solver, while a value close to \( 1 \) indicates good performance. Therefore, higher values of \emph{EcoL2} are preferable. (The proof of this Corollary follows directly from Lemma~\ref{lemma1} and Lemma~\ref{lemma2}).
\end{corollary}
\begin{wrapfigure}{r}{0.40\textwidth}
\vspace{-1em}
\begin{minipage}{0.40\textwidth}
\begin{algorithm}[H]
\caption{EcoL2 computation}
\label{Alg1}
\begin{algorithmic}[1]
\State \textbf{Input:} $\alpha$, $\beta$
\State \textbf{Output:} $\mathrm{EcoL2}$
\State Init: $\mathcal{R} \gets \text{None}$, $C_e, C_d, C_o, C_i \gets 0$, $n_{\text{infer}} \gets 0$
\If{Embodied enabled} \State Track and store $C_e$ \EndIf
\If{Development enabled} \State Run tuning $\to C_d$ \EndIf
\If{Operational enabled} \State Run final model $\to C_o$ \EndIf
\State Run $n_{\text{infer}}$ inferences, get $\mathcal{R}$ and $C_i$
\State $\mathrm{EcoL2} \gets \frac{1 - e^{\log_{\alpha}(\mathcal{R})}}{1 + \beta (C_e + C_d + C_o + C_i \cdot n_{\text{infer}})}$
\State \Return $\mathrm{EcoL2}$
\end{algorithmic}
\end{algorithm}
\end{minipage}
\vspace{-1em}
\end{wrapfigure}
\textbf{Special cases: }Developing a neural PDE solver from scratch may not always be necessary due to the availability of pre-generated data, well-tuned hyperparameters, or pre-trained models. In such scenarios, the EcoL2 metric can be generalized by excluding the carbon footprint associated with these existing resources as shown in Algorithm~\ref{Alg1}. This consideration is further discussed in Section~\ref{Sec:conclusions}.

\textbf{Measuring carbon: }This paper utilizes \texttt{CodeCarbon} to measure the emissions at each stage. \texttt{CodeCarbon} estimates emissions \(C\)\footnote{Here $C$ refers to each individual carbon component} by multiplying the energy consumed, given by power \(P\) times runtime \(t\) (in kWh), and the carbon intensity \(I\) (in kgCO\(_2\)/kWh) of the region where computations are carried out, i.e., $C = P \times t \times I$. This yields emissions in kgCO\(_2\), representing the environmental impact of the computational task. For a solver run, the pseudo-code for measuring carbon through \texttt{CodeCarbon} and computing EcoL2 is provided in supplementary material \textbf{SM\S D}. 

\textbf{Dimensional analysis: } EcoL2's numerator is a transformed relative error, hence unitless. EcoL2's denominator aggregates emissions measured in kgCO$_2$, yielding an overall unit of EcoL2 to be kgCO$_2^{-1}$. Units are omitted throughout the text for brevity, as all carbon components are expressed in kgCO$_2$.

\section{Numerical experiments}
\label{Sec:numerical_experiments}
While broadly applicable to various PDE solvers, the proposed metric is presented here in the context of PINN-based methods and neural operators. These classes are briefly described in \textbf{SM\S B}. In addition, the proposed metric is also shown to have broad applicability on diverse tasks like function approximation and symbolic regression using various ansatz such as Gaussian process regression \cite{seeger2004gaussian}, neural processes \cite{garnelo2018neural}, deep neural networks \cite{lecun2015deep} and sparse regression methods \cite{brunton2016discovering}. All software and hardware details are in \textbf{SM\S C}. Codes are provided with the supplementary files.

\subsection{Experimental setting}
To evaluate the effectiveness of the proposed metric, a diverse set of canonical PDEs are considered that represent a broad spectrum of physical phenomena: Advection, Reaction, Wave, Korteweg–de Vries (KdV), and Kuramoto–Sivashinsky (KS). The description of the canonical PDEs and the baselines for EcoL2 are presented as follows. 

\textbf{Advection:} The following hyperbolic advection equation is considered: $\frac{\partial u}{\partial t} + \beta \frac{\partial u}{\partial x} = 0$ for $\beta=10$ in $x \in [0, 2\pi],\; t \in [0, 1]$, with initial condition $u(x, 0) = \sin(x)$ and periodic boundary conditions $u(0, t) = u(2\pi, t)$, which admits the analytic solution $u(x, t) = sin(x-\beta t)$.

\textbf{Reaction:} The following reaction problem is commonly used to model chemical reactions: $\frac{\partial u}{\partial t} - \rho u(1 - u) = 0$ for $x \in [0, 2\pi]$ and $t \in [0, 1]$, with the initial condition $h(x) =$ $u(x, 0) = \exp\left(-\frac{(x - \pi)^2}{2(\pi/4)^2}\right)$ and periodic boundary condition $u(0, t) = u(2\pi, t)$. Here, $\rho$ is the reaction coefficient, set to $\rho = 5$ in this paper. The analytical solution is given by $u(x, t) = \frac{h(x)\exp(\rho t)}{h(x)\exp(\rho t) + 1 - h(x)}$.

\textbf{Wave:} The following equation models oscillatory behavior and is defined as $\frac{\partial^2 u}{\partial t^2} - \beta \frac{\partial^2 u}{\partial x^2} = 0$ for $x \in [0, 1]$, $t \in [0, 1]$, where $\beta = 3$ is the wave speed. The initial conditions are $u(x, 0) = \sin(\pi x) + \frac{1}{2} \sin(\beta \pi x)$ and $\frac{\partial u(x, 0)}{\partial t} = 0$, and the boundary conditions are $u(0, t) = u(1, t) = 0$. The equation admits the analytical solution $u(x, t) = \sin(\pi x) \cos(\sqrt{3}\pi t) + \frac{1}{2} \sin(3\pi x) \cos(3\sqrt{3}\pi t)$.

\textbf{KdV:} The equation is a third-order nonlinear PDE that describes the evolution of long, one-dimensional waves in dispersive media: $\frac{\partial u}{\partial t} + u \frac{\partial u}{\partial x} + \frac{\partial^3 u}{\partial x^3} = 0$, where \( u(x,t) \) represents the wave profile as a function of space \( x \) and time \( t \). The second term models nonlinear wave steepening, and the third term represents dispersion. The data generation details and the learning objectives for the KdV and KS equations are outlined in \textbf{SM\S G}.

\textbf{KS:} The KS equation is a nonlinear fourth-order PDE that arises in various physical systems and is given by: $\frac{\partial u}{\partial t} + u \frac{\partial u}{\partial x} + \frac{\partial^2 u}{\partial x^2} + \frac{\partial^4 u}{\partial x^4} = 0$, where \( u(x,t) \) represents the scalar field (e.g., velocity or concentration) over space \( x \) and time \( t \). The nonlinear convection term \( u \partial_x u \), the negative diffusion term \( \partial^2_x u \), and the hyperviscous term \( \partial^4_x u \) collectively result in chaotic spatiotemporal dynamics.

\textbf{PDE solvers: }The performance of EcoL2 is evaluated for six different PDE solvers over five different PDEs. Specifically, the advection, reaction, and wave PDEs are simulated using PINN-based methods (PINNs \cite{raissi2019physics}, PINNsformer \cite{zhaopinnsformer}, and SPINN \cite{cho2023separable}) and KdV and KS equations are solved through operator learning methods (DON \cite{lu2021learning}, FNO \cite{lifourier}, and CNO \cite{raonic2023convolutional}).

\textbf{Hyperparameters: }A comprehensive overview of the hyperparameters used for each PDE and their solvers are provided in \textbf{SM\S F} and \textbf{SM\S G}.

\textbf{Baselines:} This paper considers several commonly used error metrics for evaluating PDE solutions. These include relative L2 error ($\mathcal{R}$), root mean square error (RMSE), maximum error (ME), and mean absolute error (MEA). These metrics are described in \textbf{SM\S E}. The proposed metric is evaluated against these baseline metrics, and the corresponding advantages are observed in the following experiments. While baseline metrics represent numerical accuracy, they overlook an increasingly critical aspect of carbon emissions across various model training and inference stages. EcoL2 integrates carbon emissions into the evaluation, enabling a balanced assessment of accuracy and environmental impact. 

\subsection{Results and discussions}
\begin{table}[t]
\vspace{-10pt}
\centering
\caption{Performance of methods across various PDE types using different metrics}
\label{tab1}
\renewcommand{\arraystretch}{1.3}
\resizebox{\textwidth}{!}{%
\begin{tabular}{@{}c|c|c|c|c|c|c|c|c|c|c|c@{}|}
\cline{2-12}
\multicolumn{1}{c|}{} & 
\multirow{2}{*}{\diagbox{Methods}{Metrics}} & 
\multicolumn{4}{c|}{Traditional metrics} & 
\multicolumn{6}{c|}{Proposed carbon-aware metric} \\
\cline{3-12}
\multicolumn{1}{c|}{} & 
& $\mathcal{R}$ & RMSE & ME & MAE & $C_e$ & $C_d$ & $C_o$ & $C_i$ & $C$ & EcoL2 \\
\cline{2-12}

\multirow{3}{*}{\rotatebox{90}{Advection}} 
& PINNs         & 4.78e-4 & 3.38e-4 & 8.01e-3 & 2.35e-4 & - & 1.35e-2 & 8.86e-4 & 2.46e-8 & 1.44e-2 & 0.332 \\
& PINNsFormer   & 4.25e-4 & 3.01e-4 & 1.94e-3 & 2.07e-4 & - & 3.27e-1 & 2.67e-2 & 1.44e-6 & 3.54e-1 & 0.022 \\
& SPINN         & 4.01e-4 & 2.83e-4 & 8.58e-4 & 2.27e-4 & - & 5.26e-2 & 1.71e-2 & 3.59e-6 & 6.97e-2 & 0.103 \\
\cline{2-12}

\multirow{3}{*}{\rotatebox{90}{Reaction}} 
& PINNs         & 4.37e-3 & 2.95e-3 & 1.86e-2 & 1.55e-3 & - & 2.45e-3 & 3.28e-4 & 1.81e-6 & 2.77e-3 & 0.542 \\
& PINNsFormer   & 1.45e-2 & 9.78e-3 & 6.75e-2 & 4.10e-3 & - & 2.01e-1 & 9.21e-3 & 1.67e-6 & 2.10e-2 & 0.027 \\
& SPINN         & 7.61e-3 & 5.13e-3 & 3.95e-2 & 2.23e-3 & - & 5.92e-2 & 4.92e-3 & 2.84e-6 & 5.75e-2 & 0.088 \\
\cline{2-12}

\multirow{3}{*}{\rotatebox{90}{Wave}} 
& PINNs         & 7.42e-3 & 4.01e-3 & 2.99e-2 & 3.19e-3 & - & 3.60e-2 & 3.72e-2 & 2.12e-6 & 7.32e-2 & 0.078 \\
& PINNsFormer   & 2.44e-2 & 1.31e-2 & 4.89e-2 & 1.01e-2 & - & 2.84e+0 & 3.30e-1 & 3.28e-6 & 3.17e+0 & 0.002 \\
& SPINN         & 8.13e-3 & 4.38e-3 & 1.23e-2 & 3.49e-3 & - & 5.26e-2 & 3.42e-2 & 2.76e-6 & 8.68e-2 & 0.067 \\
\cline{2-12}

\multirow{3}{*}{\rotatebox{90}{KdV}} 
& DON           & 3.63e-2 & 1.57e-2 & 4.38e-2 & 9.37e-3 & 1.90e-4 & 3.74e-3 & 9.01e-4 & 2.89e-6 & 4.83e-3 & 0.346 \\
& FNO           & 7.16e-3 & 3.09e-3 & 2.47e-1 & 1.48e-3 & 3.81e-4 & 8.43e-4 & 8.88e-5 & 3.21e-6 & 1.31e-3 & 0.581 \\
& CNO           & 7.27e-3 & 3.14e-3 & 2.21e-1 & 1.68e-3 & 3.81e-4 & 7.30e-3 & 1.86e-3 & 7.56e-6 & 9.55e-3 & 0.336 \\
\cline{2-12}

\multirow{3}{*}{\rotatebox{90}{KS}} 
& DON           & 5.89e-2 & 4.57e-2 & 9.67e-1 & 2.96e-2 & 3.70e-3 & 6.10e-3 & 1.77e-3 & 2.45e-6 & 1.16e-2 & 0.213 \\
& FNO           & 1.14e-2 & 8.84e-3 & 3.90e-1 & 4.95e-3 & 3.70e-3 & 2.86e-3 & 8.38e-4 & 3.08e-6 & 7.40e-3 & 0.357 \\
& CNO           & 2.14e-2 & 1.66e-2 & 5.81e-1 & 1.11e-2 & 3.70e-3 & 1.25e-2 & 3.86e-3 & 3.17e-6 & 2.01e-2 & 0.188 \\
\cline{2-12}

\end{tabular}
}
\vspace{-10pt}
\end{table}

\textbf{Performance on PINN family: }Table~\ref{tab1} presents the performance of PINNs, PINNsFormer, and SPINN across advection, reaction, and wave PDEs, using both error-based baseline metrics and our proposed carbon-aware EcoL2 metric for $\alpha=\beta=100$. For developmental stage, the experiment conducted tuning over three residual weights ($\lambda_r$) and four model configurations for each method and PDE as detailed in \textbf{SM\S F}. 

Across all PDEs, $C_d$ emerges as the dominant contributor to total emissions ($C$), highlighting the intensive cost of hyperparameter tuning. Additionally $C_o$ is nearly equivalent to $C_d$, particularly in the wave PDE case. This is attributed to the complexity of the problem, for which a limited number of training epochs are performed when tuning the hyperparameters in the developmental stage. For the advection PDE, both PINNsFormer and SPINN achieve comparable accuracy in terms of traditional metrics: the relative error ($\mathcal{R}$) is $4.25e-4$ for PINNsFormer and $4.01e-4$ for SPINN. At the same time, RMSE and MAE also remain closely matched. However, PINNsFormer has an EcoL2 score of just 0.022, while SPINN scores 0.103—almost five times higher. This disparity is primarily due to the significantly higher developmental ($C_d = 3.27e-1$) and operational carbon ($C_o = 2.67e-2$) costs associated with the transformer-based architecture in PINNsFormer, compared to SPINN's lower $C_d = 5.26e-2$ and $C_o = 1.71e-2$. Despite their similar accuracy, EcoL2 reveals the environmental inefficiency of PINNsFormer—a difference that error only metrics fail to capture. These trends are further visualized in Figure~\ref{fig1} for the advection equation in \textbf{SM\S F} for the reaction and wave equations. 

\textbf{Performance on higher-dimensional problem: }The performance of EcoL2 is further analyzed for high-dimensional thermal diffusion equations. Neural network-based ansatzs are considered advantageous over traditional methods by mitigating the curse of dimensionality \cite{hu2024tackling, mishra2021physics}. The rationale for this experiment is that while these methods provide comparable accuracy across different dimensions, the carbon emissions vary when increasing the dimensionality. The experimental details and results discussed in \textbf{SM\S F} showcase the performance of EcoL2 on SPINN, the state-of-the-art method for higher dimensional problems. 

\textbf{Performance on neural operator family: }Table~\ref{tab1} presents the performance of three operator learning methods—DON, FNO, and CNO—on the KdV and KS PDEs, evaluated using both baseline error metrics and EcoL2. Unlike PINN-based methods, operator learning approaches are data-driven supervised method and require training data generation through simulations contributing to $C_e$. 

Across both PDEs, the EcoL2 metric showcases the trade-offs between numerical accuracy and environmental cost. For instance, on the KS equation, CNO achieves lower relative error ($2.14e-2$) than DON, which has $\mathcal{R}=5.89e-2$. However, DON has a low total carbon cost ($C=1.16e-2$), leading to an EcoL2 score of 0.213. In contrast, CNO has a slightly higher total carbon cost ($C = 2.01e-2$), resulting in a lower EcoL2 score of 0.188. This illustrates that while traditional metrics might seem comparable, EcoL2 provides a more holistic view by incorporating emissions. 

\subsection{Ablation studies}

\begin{figure}[t]
    \centering
    \vspace{-10pt}
    \includegraphics[width=0.20\textwidth]{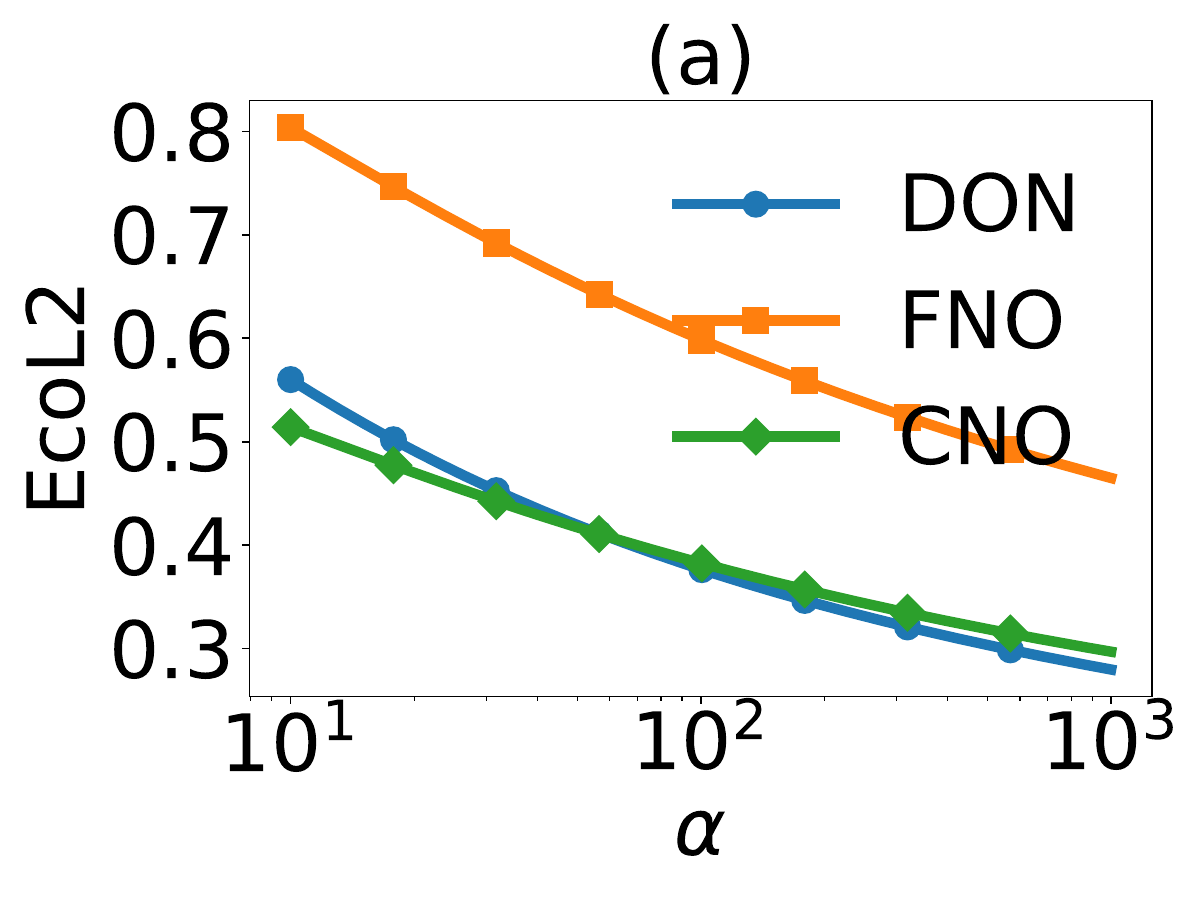}
    \includegraphics[width=0.20\textwidth]{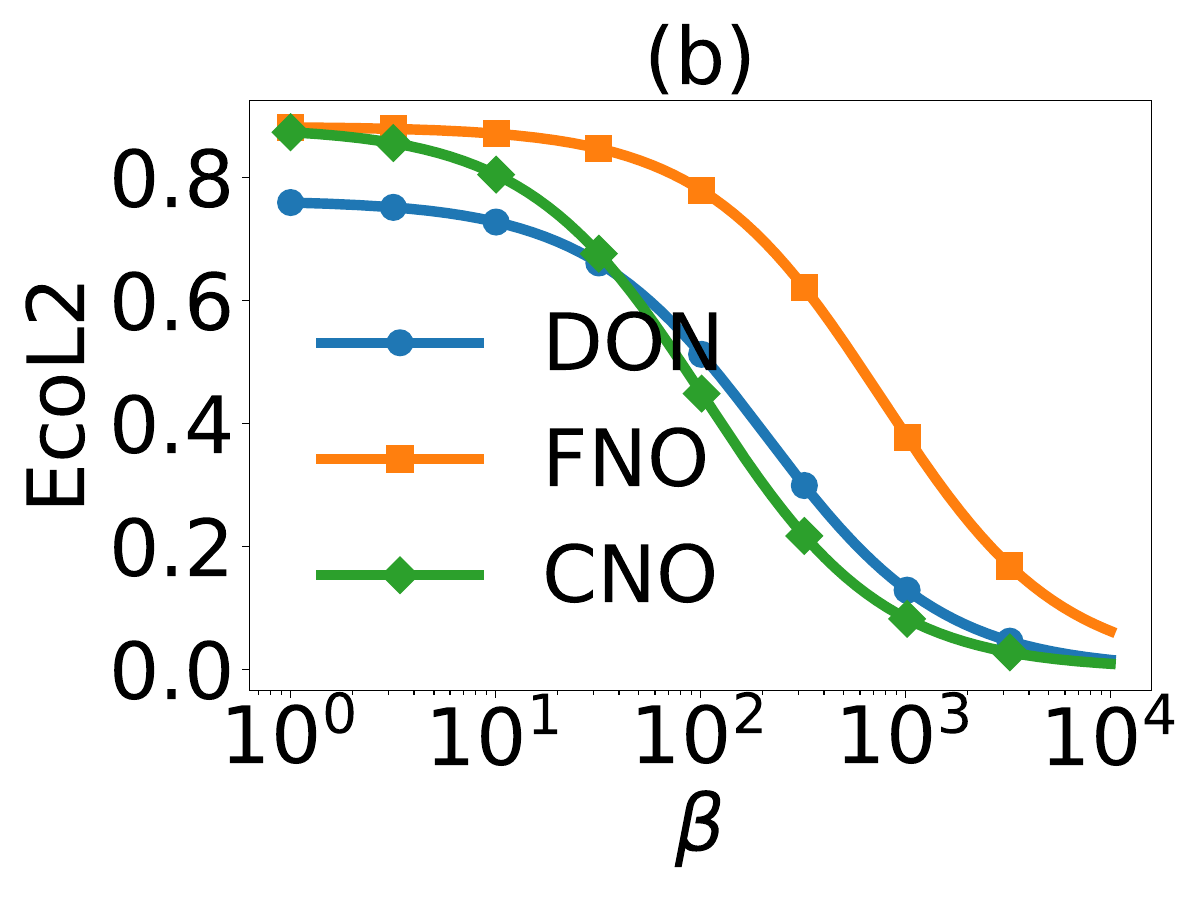}
    \includegraphics[width=0.20\textwidth]{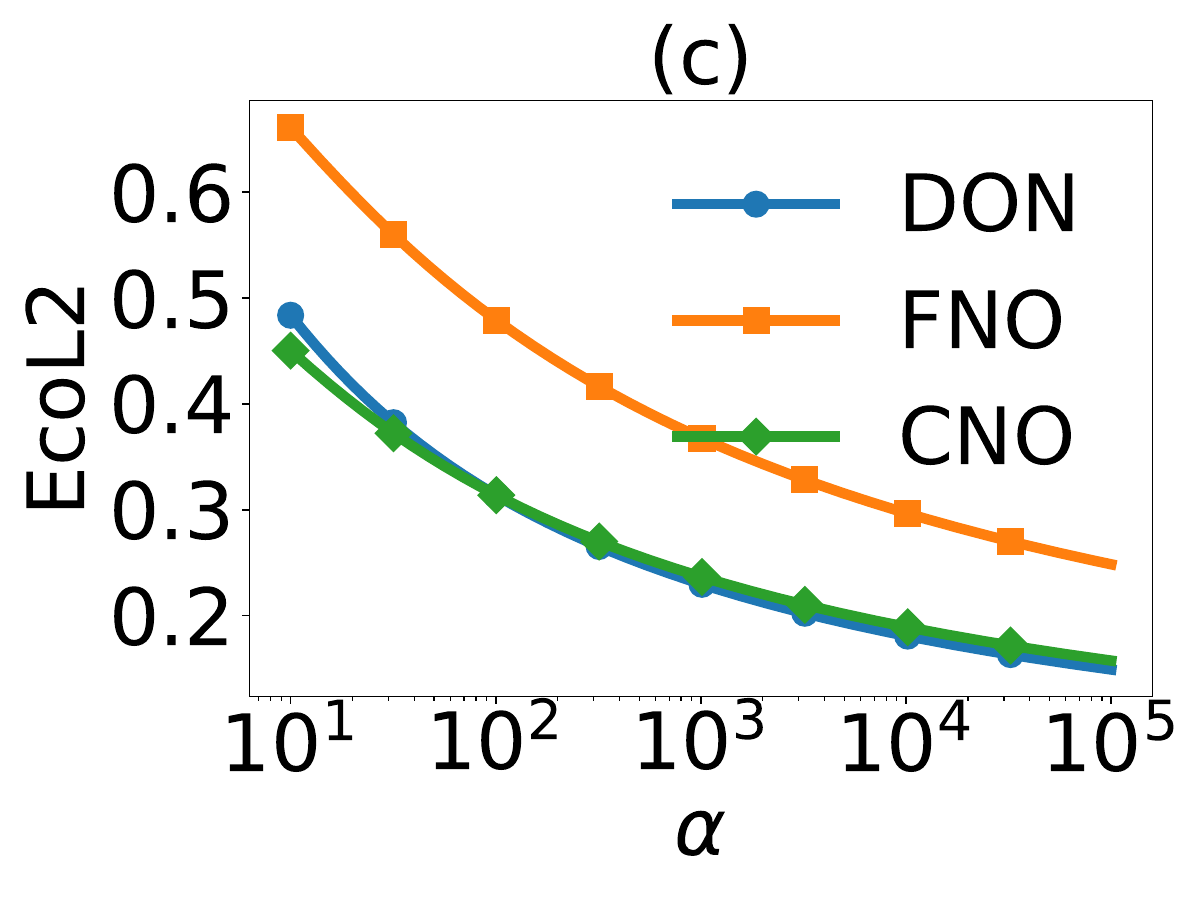}
    \includegraphics[width=0.20\textwidth]{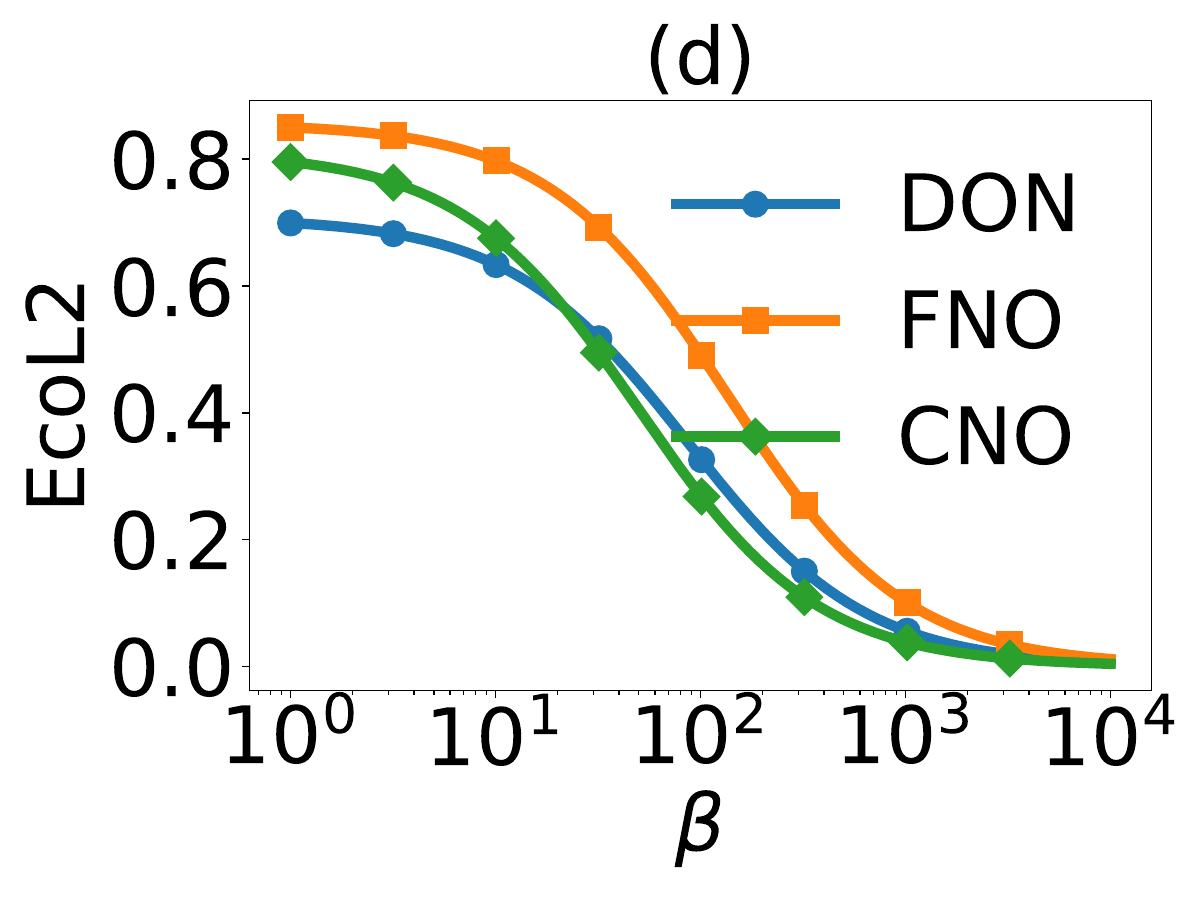}
    
    \caption{EcoL2's adaptive performance: KdV (a, b) and KS (c, d) PDEs for varying $\alpha$ and $\beta$}
    \label{fig5}
    \vspace{-15pt}
\end{figure}

\textbf{Variation in $\alpha$ and $\beta$: }We perform an ablation study by varying the hyperparameters $\alpha$ and $\beta$ used in the metric’s formulation. These hyperparameters control the weighting between accuracy and carbon impact in the EcoL2 metric. This experiment is performed across the three operator learning methods on the KdV and KS equations. For the KdV equation, we fix $\beta = 75$ and vary $\alpha$, as illustrated in Figure~\ref{fig5}(a). Conversely, in Figure~\ref{fig5}(b), we fix $\alpha = 10$ and vary $\beta$. 

In Figure~\ref{fig5}(a), for lower values of $\alpha$—where sustainability is prioritized over accuracy—the DON outperforms CNO. However, the CNO method becomes preferable as $\alpha$ increases, as desirable from the metric formulation. This shift in preference illustrates how EcoL2 would adapt to domain-specific requirements. In contrast, a single-score metric such as $\mathcal{R}$ fails to capture these dynamics, offering a static evaluation that neglects the environmental cost-performance trade-off. Similarly, Figure~\ref{fig5}(b) shows how increasing $\beta$, which reflects greater emphasis on sustainability, can alter optimal method selection. These results illustrate that EcoL2 enables a more comprehensive and application-aligned model assessment by explicitly parameterizing accuracy and sustainability.

Similar results are obtained for the KS equation as shown in Figure~\ref{fig5}(c-d). Figure~\ref{fig5}(c) shows the variation in EcoL2 with $\alpha$ when $\beta$ is fixed at 40, while Figure~\ref{fig5}(d) presents the results of varying $\beta$ with $\alpha$ held constant at 10.

\begin{wrapfigure}{r}{0.40\textwidth}
    \centering
    \vspace{-10pt}
    \includegraphics[width=0.40\textwidth]{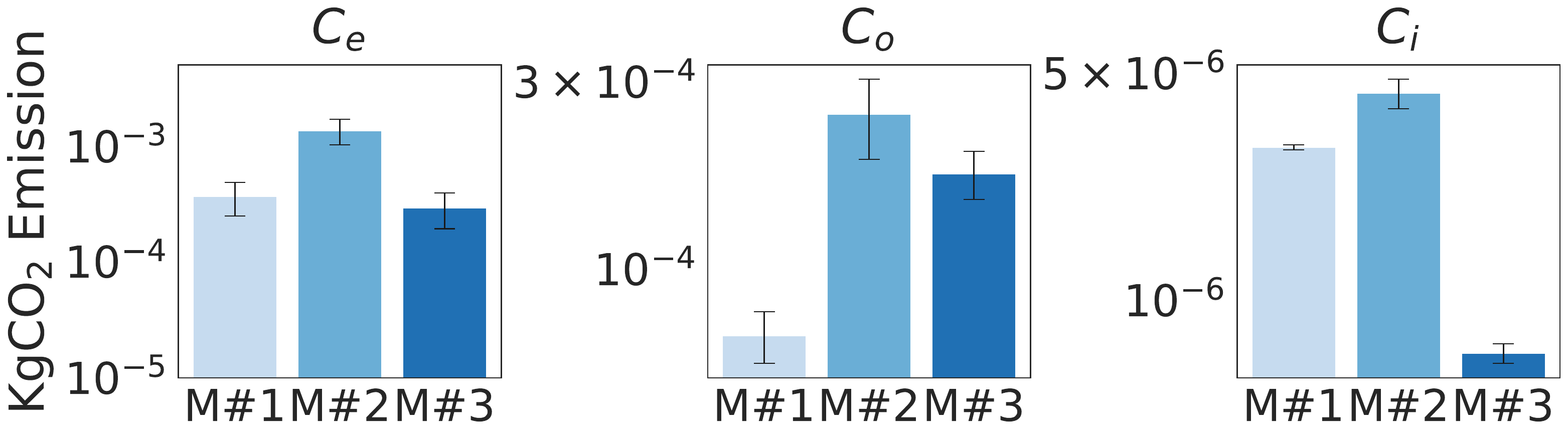}
    \caption{Impact of different machines}
    \label{fig6}
    \vspace{-10pt}
\end{wrapfigure}

\textbf{Impact of hardware choice on emissions: }EcoL2 metric also considers the impact of hardware choice by considering the energy consumed depending on power and run time ($P \times t$) as detailed in Section~\ref{Sec:proposed_metric}. This experiment presents an ablation study for the KdV equation simulated using FNO, showcasing the variation in carbon emissions across three different machines: M\#1, M\#2, and M\#3. The hardware details are provided in \textbf{SM\S C}. The code is run five times on each machine, and mean and standard deviations are reported in Figure~\ref{fig6} for three stages of model development ($C_e, C_o, C_i$). The figure illustrates how hardware choice can influence the carbon footprint of neural PDE solvers, even when the algorithm and task remain unchanged. Machine M\#2  exhibits the highest carbon emissions. Machine M\#3 shows relatively lower emissions for the operational stage, indicating that it is suitable for model training. This analysis shows the importance of EcoL2 over baselines in considering hardware-specific carbon variability when benchmarking neural PDE solvers.
  
\begin{wrapfigure}{r}{0.50\textwidth}
    \centering
    \vspace{-10pt}
    \includegraphics[width=0.50\textwidth]{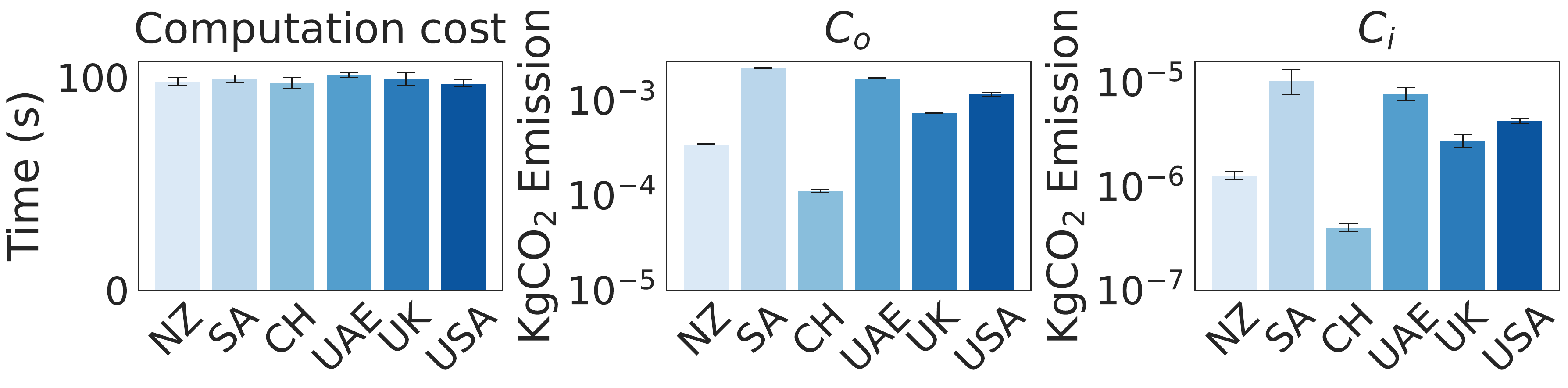}
    \caption{Country-wise variations}
    \label{fig7}
    \vspace{-10pt}
\end{wrapfigure}

\textbf{Cross-regional emissions: }EcoL2 metric considers the carbon intensity ($I$) of the region where the solver is run (presented briefly in \textbf{SM\S I}) as illustrated in Section~\ref{Sec:proposed_metric}. This experiment presents an ablation study for the KS equation simulated using FNO, showcasing the variation in carbon emissions across six different countries: New Zealand (NZ), South Africa (SA), Switzerland (CH), UAE, the UK, and the USA. Figure~\ref{fig7} compares the carbon footprint and computational cost of running the same model for the same task in different countries. The data shown represent averages over five independent runs per country, and the figure includes standard deviation bars to reflect run-to-run variability. The analysis is divided into total computational time, operational carbon emissions ($C_o$), and inference carbon emissions ($C_i$). While computational time remains relatively consistent across regions, suggesting comparable performance, the corresponding carbon emissions vary significantly. For instance, all the countries have identical runtime; however, SA and UAE, for instance, have substantially higher $C_o$ and $C_i$ values, indicating a larger environmental impact for the same task. Conversely, countries like Switzerland and New Zealand, with cleaner energy sources, yield much lower emissions. This analysis shows that carbon emissions and computational cost are related but fundamentally different metrics. This analysis shows the importance of EcoL2 in considering cross-regional impacts, which baseline metrics ignore entirely.

\begin{wraptable}{r}{0.38\textwidth}
\vspace{-1.5em}
\centering
\caption{Broader applicability}
\label{broadimpact}
\renewcommand{\arraystretch}{1.05}
\small
\begin{tabular}{@{}lccc@{}}
\toprule
Method & $\mathcal{R}$ & $C$ & EcoL2 \\
\midrule
\multicolumn{4}{c}{Function Approximation} \\
GP   & 4.1e-5   & 1.10e-6  & 0.887  \\
NP   & 5.4e-5   & 3.33e-7  & 0.881  \\
DNN  & 1.67e-4  & 1.58e-6  & 0.849  \\
\midrule
\multicolumn{4}{c}{Symbolic Regression} \\
SINDy & 7.79e-7  & 8.94e-7 & 0.953  \\
\bottomrule
\end{tabular}
\end{wraptable}
\subsection{Broader impact} 
Beyond PDE simulation, the metric is applicable globally for tasks which involve model selection based on accuracy. Here, we show the broader applicability of EcoL2 for two tasks beyond PDE simulations, namely, function approximation and symbolic regression. The methods applied for function approximation include deep neural networks (DNNs), Gaussian process (GP) regression, and neural processes (NPs). A benchmark mathematical function is approximated through these frameworks using a supervised data-driven learning. Another task, namely symbolic regression for dynamical system discovery is also performed using the SINDy method. The details of the experiments and model configurations are discussed in \textbf{SM\S H}. Brief results are shown in Table~\ref{broadimpact}. Notably, GP and NP achieve similar order of accuracy in function approximation, but their overall carbon emissions differ by nearly an order of magnitude. The results illustrates the need for sustainable evaluation metrics across scientific machine learning.

\section{Related work}
\label{Sec:related}
Recently some works have raised critical ethical and sustainability concerns about the increasing carbon footprint of core AI models while neither discussing neural PDE solvers nor proposing a performant metric on how to balance model accuracy with carbon positivity \cite{luccioni2020morality, dodge2022measuring, lacoste2019quantifying}. For instance, training costs of state-of-the-art deep learning models have increased by over 300,000$\times$ since 2012 \cite{schwartz2020green}. This impact extends across the AI lifecycle, from training and frequent fine-tuning and deployment \cite{wang2023energy}. For instance, large models like BLOOM \cite{le2023bloom} have been shown to emit tens of tonnes of carbon during training alone \cite{luccioni2023estimating}, while generative AI models like GPT-4 \cite{achiam2023gpt} require continuous energy and cooling needs \cite{mitExplainedGenerative}. In addition, these emissions also implicitly contain the hidden costs of hardware production and rare-earth mining \cite{dhar2020carbon}. 

This paper is in a similar spirit but focused on scientific machine learning models with discussions on neural PDE solvers. The development of neural PDE solvers has roughly followed the same chronological trend as core AI development, from multi-layer perceptrons \cite{raissi2019physics, lu2021learning} to convolutional architectures \cite{fang2021high, lifourier, raonic2023convolutional}, transformer-based \cite{zhaopinnsformer, hao2023gnot} and recently large language model-inspired \cite{yang2023context, zheng2024alias}, and Swin Transformer-based foundation models \cite{herde2024poseidon}. However, there has been no discussion of the carbon footprint of neural PDE solvers, and given this topic is an emerging field, it is worthwhile to account for its carbon footprint before the situation becomes alarming.

This paper introduces different stages of carbon footprint in neural PDE solvers. The work leverages notions of embodied and operational carbon from the structural engineering domain, where these concepts are used frequently to quantify infrastructure carbon \cite{ghorbany2025automating}. Recently, these concepts have been used in the AI domain \cite{luccioni2023estimating}. This paper extends the discussion in the context of neural PDE solvers and identifies four critical sources of carbon footprint. However, the proposed notions differ from their usual definitions because of the domain-specific challenges, as presented in Section~\ref{Sec:life_cycle}. 

Recently, several emission tracking tools have been developed, including but not limited to \texttt{CodeCarbon} \cite{schmidt2021codecarbon}, \texttt{Carbontracker} \cite{anthony2020carbontracker}, \texttt{Experiment Impact Tracker} \cite{henderson2020towards}, \texttt{Eco2AI} \cite{budennyy2022eco2ai}, and \texttt{Green Algorithms} \cite{lannelongue2021green}. The current paper leverages this line of research and couples it with model accuracy to propose a performant metric to assist in a holistic assessment of neural PDE solvers.

\section{Discussion and conclusion}
\label{Sec:conclusions}
\textbf{Discussion: }The results demonstrate that models with similar accuracy can differ in their carbon footprint, suggesting that measuring accuracy alone is insufficient for sustainable model selection. By incorporating environmental considerations in model evaluation, EcoL2 presents a more holistic overview of the method. The formulation of EcoL2 is adaptive, which allows for emphasis on either accuracy or emissions through tunable weighting parameters. That is, the metric applies across a broad range of applications, such as high-stakes engineering design, to large-scale simulations in climate science. The results also demonstrate the critical role of hardware choices in contributing to carbon emissions. By applying our framework across a range of physics-informed machine learning and operator learning methods, the metric is shown to be a potent tool for the domain. In addition, the embodied emissions, which represent the contribution of numerical methods towards carbon emissions, indicate their significant yet underexplored role in contributing to carbon footprint. The results highlight the need for metrics like EcoL2 in scientific computing beyond neural network-based methods. Furthermore, we emphasize that computational cost is not synonymous with carbon emissions, as they differ geographically, an often overlooked distinction. Our study thus provides a first step towards quantifying the environmental cost of scientific AI workflows and establishes a principled foundation for quantifying low-carbon, high-performance solvers. 

\textbf{Conclusion: }We pioneer the concept of carbon accountability in scientific machine learning by proposing a novel EcoL2 metric that explicitly balances model accuracy with carbon emissions. This represents a step toward addressing the often-overlooked environmental cost of training deep learning models for simulating PDEs. We believe this is an essential direction for ensuring the sustainable development of neural PDE solvers and operator learning models as they scale in complexity and real-world deployment. By validating it across various PDE-solving frameworks, we demonstrated that EcoL2 offers a meaningful trade-off between predictive performance and environmental impact. In addition to experimental validation, we analyze the proposed metric mathematically. While our primary focus was on PDEs, we emphasize that the EcoL2 metric generalizes to broader scientific tasks such as function approximation and symbolic regression. We view this work as a foundation for further research in building environmentally conscious scientific machine learning systems.

\textbf{Limitations: }The EcoL2 metric is a step towards balancing accuracy and environmental impact in scientific machine learning, but it also presents a few limitations. First, estimating emissions from datasets already stored and accessed via cloud services is inherently problematic since the energy cost of their original generation, preprocessing, and storage is often unknown or unrecorded. This makes it challenging to account for the carbon footprint associated with their use entirely. Second, when multiple learning methods are developed/validated on the same dataset, there is considerable overlap in data usage, which complicates the attribution of emissions during the developmental phase. While accounting for emissions during training and deployment is easier, a robust mechanism for disaggregating shared embodied emissions is a future work. This is more a limitation of the current emission accounting methodologies rather than of EcoL2 itself, and we expect that future work will address this through better provenance tracking and emission partitioning protocols in shared-data machine learning workflows.


\newpage
\appendix
\section{Organization of supplementary material}
The supplementary material is organized as follows. Section~\ref{sm_sec:B} provides a brief overview of classes of neural PDE solvers utilized in this paper, namely PINN-based models and neural operators. Section~\ref{sm_sec:C} presents the reproducibility statement, detailing all experiments' hardware and software configurations. Section~\ref{sm_sec:D} provides the pseudo-code for the proposed EcoL2 metric. Section~\ref{sm_sec:E} details the baseline error metrics used in the paper. Section~\ref{sm_sec:F} presents the hyperparameters and tuning results for the PINN-based methods. Section~\ref{sm_sec:G} presents the neural operators' hyperparameter settings and tuning results. Section~\ref{sm_sec:H} presents the details of additional experiments demonstrating the broader applicability of the EcoL2 metric. Finally, Section~\ref{sm_sec:I} details the country-wise carbon intensity data used in this paper. 

\section{Neural PDE solvers}
\label{sm_sec:B}
Early works in simulating differential equations using neural networks date back to the 1990s \cite{dissanayake1994neural, lagaris1998artificial}. However, the field has gained significant traction over the past decade, owing to advances in computational resources, software frameworks, and autodifferentiation. Two classes of neural PDE solvers considered in this work are PINN-based models and neural operators.

\textbf{PINN-based models: }PINNs \cite{raissi2019physics} and its advanced variants (such as PINNsFormer \cite{zhaopinnsformer} and SPINN \cite{cho2023separable}) are deep learning-based architectures that leverage the physical equations during training to guide the network toward accurate solutions, even with limited available data. Consider the abstract PDE ($F$), along with boundary operator ($G$) in computational domain $\Omega$, and its boundary $\partial\Omega$:  
\begin{align}
F(u) &:= F(x, u, \nabla u, \dots) = f(x) \quad && x \in \Omega, \\
G(u) &:= G(x, u, \nabla u, \dots) = g(x) \quad && x \in \partial\Omega,
\end{align}
where $x$ is the independent variable, $u$ is the quantity of interest, $f(x)$ and $g(x)$ are source function and boundary condition, respectively. The network prediction $u_{\theta}$ is used in the loss function:
\begin{equation}
\mathcal{L}_{\text{PINN}}(\theta) = \int_{\Omega} \left\| F(u_{\theta}) - f(x) \right\|^2 \, dx 
+ \lambda_r \int_{\partial \Omega} \left\| G(u_{\theta}) - g(x) \right\|^2 \, dx.
\label{eq:pinn-loss}
\end{equation}
The learning goal is to obtain optimal trainable parameters ($\theta$) that minimize~\eqref{eq:pinn-loss} using a residual parameter ($\lambda_r$) and suitable quadrature rule to estimate the integrals \cite{mishra2023estimates}.

\textbf{Neural operators: }PINN-based models are limited to simulating a particular instance(s) of PDE, similar to conventional numerical solvers. Neural operators, like DON \cite{lu2021learning}, FNO \cite{lifourier}, and CNO \cite{raonic2023convolutional}, are supervised data-driven methods that simulate PDEs which are parametrized in source terms, boundary conditions or PDE coefficients. By learning mappings between the infinite-dimensional parametrized input space (for instance, $f(x) \in \mathcal{F}$) and output space of quantity of interest ($u(x) \in \mathcal{U}$), neural operators generalize and provide inference across unseen parametrized space. Precisely, the goal is to learn an operator \( \mathcal{G}_\theta: \mathcal{F} \rightarrow \mathcal{U} \) by minimizing:
\begin{equation}
\mathcal{L}_{\text{NO}}(\theta) = \int_{f\sim \mathcal{F}} \left\| \mathcal{G}_\theta(f) - u(f) \right\|^2 \, df.
\label{eq:no-loss}
\end{equation}

\textbf{Other solvers: }Several other classes of neural PDE solvers, such as physics-informed neural operators \cite{wang2021learning, goswami2023physics} and foundation models \cite{herde2024poseidon} have been developed. See \cite{huang2025partial} for a recent review of neural PDE solvers. The proposed measure of carbon emission and the metric would also apply to these classes of solvers. However, for brevity, they are not considered in this paper.

\section{Reproducibility statement}
\label{sm_sec:C}

To ensure reproducibility, the code is provided in SM and will later be open sourced at GitHub upon publication. The used model hyperparameters and dataset are reported in detail in Section~\ref{sm_sec:F} and Section~\ref{sm_sec:G}. Additionally, the pseudo-code for our proposed metric, EcoL2, is included in Section~\ref{sm_sec:D}. 

The experiments are run on four different machines: M\#1,  M\#2,  M\#3, and  M\#4. Specifically, neural operator experiments presented in Section 4.2 and ablation on varying $\alpha$ and $\beta$ are run on  M\#1. All the codes of PINN-based methods are run on  M\#4. Machine  M\#2 and  M\#3 are used for ablation experiments on assessing the impact of hardware choices and cross-regional emissions. All experiments examining the impact of hardware choices and cross-regional emissions have been repeated five times. 

As existing software assets, this paper utilized open source available codes for all the methods\textemdash PINNs \cite{raissi2019physics, githubbeam}, PINNsFormer \cite{zhaopinnsformer, githubAdityaLabpinnsformer}, SPINN \cite{cho2023separable, githubStnamjefSPINN}, DON \cite{lu2021learning, githubRajbrownAPMA2070}, FNO \cite{lifourier, githubNeuraloperatorneuraloperator}, CNO \cite{raonic2023convolutional, githubCamlabethzConvolutionalNeuralOperator}. The methods are implemented in Python-based training frameworks in PyTorch \cite{paszke2017automatic} and JAX \cite{jax2018github}. Matplotlib \cite{hunter2007matplotlib} is used for plotting, and NumPy \cite{harris2020array}  and Pandas \cite{reback2020pandas} for data handling. For data generation in the case of neural operators, SciPy \cite{2020SciPy-NMeth} is used in the public code \cite{brandstetter2022lie, githubBrandstetterjohannesLPSDA}. The usage of these assets is further described in Section~\ref{sm_sec:F} and Section~\ref{sm_sec:G}. The hardware and software environments used to perform the experiments on these machines are as follows,

Machine M\#1: UBUNTU 20.04.6 LTS, PYTHON 3.9.7, NUMPY 1.23.5, SCIPY 1.13.0, MATPLOTLIB 3.4.3, PYTORCH 2.4.0, CUDA 11.7, NVIDIA Driver 515.105.01, i7 CPU, and NVIDIA GEFORCE RTX 3080.

Machine M\#2: UBUNTU 20.04.4 LTS, PYTHON 3.9.12, NUMPY 2.0.2, SCIPY 1.13.1, MATPLOTLIB 3.9.4, PYTORCH 1.12.1, CUDA 11.7, NVIDIA Driver 535.230.02, i7-8700K CPU, and NVIDIA GEFORCE GTX 1080 Ti.

Machine M\#3: UBUNTU 24.04.2 LTS, PYTHON 3.12.7, NUMPY 1.26.4, SCIPY 1.12.0, MATPLOTLIB 3.9.2, PYTORCH 2.6.0, i5-8350U CPU.

Machine M\#4: WINDOWS 10 OS, PYTHON 3.12.7, NUMPY 1.26.4, SCIPY 1.13.1, MATPLOTLIB 3.9.2, PYTORCH 2.7.0, JAX 0.5.0, CUDA 11.8, AMD Ryzen 5 5600X CPU, and NVIDIA GEFORCE RTX 3060 Ti.

\section{Pseudo-codes}
\label{sm_sec:D}
This section presents six pseudo-codes to compute and track carbon emissions and compute EcoL2 scores. The pseudo-codes shown in Figures~\ref{code_ce} to~\ref{code_full} elaborate Algorithm 1 and detail the method to compute the carbon emissions throughout the lifecycle of neural PDE solvers. To ensure the proper execution of these codes, an \texttt{Emissions} folder with the appropriate subdirectories (\texttt{Embodied/}, \texttt{Developmental/}, \texttt{Operational/}, and \texttt{Inference/}) must be created in the working directory.

Figure~\ref{code_ce} computes the embodied carbon by solving PDE using a numerical solver, saving the data and the associated emissions. Figure~\ref{code_cd} loads the saved data and applies deep learning to perform hyperparameter tuning, capturing the developmental carbon. Figure~\ref{code_co} utilizes the best hyperparameters to train the final model and record the operational carbon. Figure~\ref{code_ci} loads the trained model and performs inference, logging the inference carbon and associated errors. Finally, Figure~\ref{code_full} integrates emissions from all phases and the inference error to compute the EcoL2 score, a holistic measure combining performance and sustainability. Figures~\ref{code_general} show how the proposed EcoL2 metric can be used in general for any deep learning model or even for scientific computing codes.

\begin{figure}[H]
\centering
\begin{lstlisting}[style=pythonpseudo, language=Python]
import numpy as np
# Import the emissions tracker from CodeCarbon
from codecarbon import EmissionsTracker

# Initialize the emissions tracker to track embodied carbon
tracker = EmissionsTracker(
    project_name="Embodied Carbon",
    output_file="embodied.csv",  # Emission data saved here
    output_dir="./Emissions/Embodied/",  # Emissions folder path
    save_to_file=True,
    experiment_id="C_e"
)

# Start tracking emissions
tracker.start()

# ---------------------------------------------
# Numerical solver block
# Code of numerical solver for solving the PDE
# ---------------------------------------------

# ---------------------------------------------
# Save the computed data
# ---------------------------------------------

# Stop tracking emissions and save the emission data to file
tracker.stop()
\end{lstlisting}
\caption{Pseudo-code for computing $C_e$.}
\label{code_ce}
\end{figure}

\begin{figure}[H]
\centering
\begin{lstlisting}[style=pythonpseudo, language=Python]
import numpy as np
# Import the emissions tracker from CodeCarbon
from codecarbon import EmissionsTracker

# Initialize the emissions tracker to track developmental carbon
tracker = EmissionsTracker(
    project_name="Developmental Carbon",
    output_file="developmental.csv",  # Emission data saved here
    output_dir="./Emissions/Developmental/",  # Emissions folder path
    save_to_file=True,
    experiment_id="C_d"
)

# ---------------------------------------------
# Load the data saved from the numerical solver
# ---------------------------------------------

# Start tracking emissions
tracker.start()

# ---------------------------------------------
# Model and hyperparameter tuning Block
# Save the best hyperparameters
# ---------------------------------------------

# Stop tracking emissions and save the emission data to file
tracker.stop()
\end{lstlisting}
\caption{Pseudo-code for computing $C_d$.}
\label{code_cd}
\end{figure}

\begin{figure}[H]
\centering
\begin{lstlisting}[style=pythonpseudo, language=Python]
import numpy as np
# Import the emissions tracker from CodeCarbon
from codecarbon import EmissionsTracker

# Initialize the emissions tracker to track developmental carbon
tracker = EmissionsTracker(
    project_name="Operational Carbon",
    output_file="operational.csv",  # Emission data saved here
    output_dir="./Emissions/Operational/",  # Emissions folder path
    save_to_file=True,
    experiment_id="C_o"
)

# ---------------------------------------------
# Load the best hyperparameters
# ---------------------------------------------

# Start tracking emissions
tracker.start()

# ---------------------------------------------
# Perform the training and save the model 
# ---------------------------------------------

# Stop tracking emissions and save the emission data to file
tracker.stop()
\end{lstlisting}
\caption{Pseudo-code for computing $C_o$.}
\label{code_co}
\end{figure}

\begin{figure}[H]
\centering
\begin{lstlisting}[style=pythonpseudo, language=Python]
import numpy as np
# Import the emissions tracker from CodeCarbon
from codecarbon import EmissionsTracker

# Initialize the emissions tracker to track developmental carbon
tracker = EmissionsTracker(
    project_name="Inference Carbon",
    output_file="inference.csv",  # Emission data saved here
    output_dir="./Emissions/Inference/",  # Emissions folder path
    save_to_file=True,
    experiment_id="C_i"
)

# ---------------------------------------------
# Load the trained model
# ---------------------------------------------

# Start tracking emissions
tracker.start()

# ---------------------------------------------
# Infer the Model 
# ---------------------------------------------

# Stop tracking emissions and save the emission data to file
tracker.stop()

# ---------------------------------------------
# Save the obatined errors
# ---------------------------------------------
\end{lstlisting}
\caption{Pseudo-code for computing $C_i$.}
\label{code_ci}
\end{figure}

\begin{figure}[H]
\centering
\begin{lstlisting}[style=pythonpseudo, language=Python]
import numpy as np
import pandas as pd
from math import log, exp

# ---------------------------------------------
# Load the error from inference phase
# Replace the path and format as per your actual saved error
# ---------------------------------------------
R = np.load('./results/inference_error.npy')  # Assumes scalar value

# ---------------------------------------------
# Load emissions from all carbon phases
# ---------------------------------------------
def load_emission(filepath):
    df = pd.read_csv(filepath)
    return df.iloc[1, 5]  # 2nd row, 6th column, as per codecarbon saving style (0-indexed)

c_e = load_emission('./Emissions/Embodied/embodied.csv')
c_d = load_emission('./Emissions/Developmental/developmental.csv')
c_o = load_emission('./Emissions/Operational/operational.csv')
c_i = load_emission('./Emissions/Inference/inference.csv')

# ---------------------------------------------
# Define the hyperparameters
# ---------------------------------------------
alpha = 100
beta = 100
n_infer = 1  # Number of inferences

# ---------------------------------------------
# EcoL2 calculation function
# ---------------------------------------------
def compute_ecol2(R, c_e, c_d, c_o, c_i, alpha, beta, n_infer):
    c = c_e + c_d + c_o + c_i * n_infer
    ecol2 = (1 - exp(log(R) / log(alpha))) / (1 + beta * c)
    return ecol2, c

# ---------------------------------------------
# Compute and display EcoL2 and total carbon
# ---------------------------------------------
ecol2, total_c = compute_ecol2(R, c_e, c_d, c_o, c_i, alpha, beta, n_infer)
print("EcoL2 Score:", ecol2)
print("Total Carbon Emission:", total_c)
\end{lstlisting}
\caption{Pseudo-code for computing EcoL2 score for neural PDE solvers.}
\label{code_full}
\end{figure}

\begin{figure}[H]
\centering
\begin{lstlisting}[style=pythonpseudo, language=Python]
import numpy as np
import pandas as pd
from math import log, exp

# Initialize the emissions tracker to track developmental carbon
tracker = EmissionsTracker(
    project_name="Carbon",
    output_file="carbon.csv",  # Emission data saved here
    output_dir="./Emissions/",  # Emissions folder path
    save_to_file=True,
    experiment_id="C"
)

tracker.start()

# ---------------------------------------------
# Train and test the model (or run the numerical solver)
# ---------------------------------------------

tracker.stop()

# ---------------------------------------------
# Compute the error, say R
# ---------------------------------------------

# Load emissions 

def load_emission(filepath):
    df = pd.read_csv(filepath)
    return df.iloc[1, 5]  # 2nd row, 6th column, as per codecarbon saving style (0-indexed)

C = load_emission('./Emissions/carbon.csv')

# ---------------------------------------------
# Define the hyperparameters
# ---------------------------------------------
alpha = 100
beta = 100
n_infer = 1  # Number of inferences

# ---------------------------------------------
# EcoL2 calculation function
# ---------------------------------------------
def compute_ecol2(R, C, alpha, beta):
    ecol2 = (1 - exp(log(R) / log(alpha))) / (1 + beta * C)
    return ecol2

# ---------------------------------------------
# Compute and display EcoL2 and total carbon
# ---------------------------------------------
ecol2 = compute_ecol2(R, C, alpha, beta)
print("EcoL2 Score:", ecol2)
\end{lstlisting}
\caption{Pseudo-code for computing EcoL2 for a deep learning model or a numerical solver.}
\label{code_general}
\end{figure}

\section{Baseline metrics}
\label{sm_sec:E}
The baseline error metrics used for evaluating PDE solutions in the paper include relative L2 error ($\mathcal{R}$), root mean square error (RMSE), maximum error (ME), and mean absolute error (MEA). Considering $u_i \text{ is the true value at point } i, \hat{u}_i^{(r)} \text{ is the prediction from run } r, r=1, \ldots, R$, and $\bar{\hat{u}}_i = \frac{1}{R} \sum_{r=1}^{R} \hat{u}_i^{(r)} \text{is the mean prediction over } R \text{ runs}$, the metrics are defined as follows,

\[
\mathcal{R} = \frac{\left( \sum_{i=1}^{n} (\bar{\hat{u}}_i - u_i)^2 \right)^{1/2}}{\left( \sum_{i=1}^{n} u_i^2 \right)^{1/2}}
\]

\[
\text{RMSE} = \left( \frac{1}{n} \sum_{i=1}^{n} (\bar{\hat{u}}_i - u_i)^2 \right)^{1/2}
\]

\[
\text{ME} = \max_{i} |\bar{\hat{u}}_i - u_i|
\]

\[
\text{MAE} = \frac{1}{n} \sum_{i=1}^{n} | \bar{\hat{u}}_i - u_i |
\]

\section{Experimental details for PINN-based methods}
\label{sm_sec:F}
This section presents the experimental details for PINN-based methods, PINNs, PINNsFormer, and SPINN for the advection, reaction, and wave equations. The subsections~\ref{dev_pinn},~\ref{dev_pf}, and~\ref{dev_spinn} present the developmental stage details, while the operational stage details are presented in~\ref{operational}. The comparison trade-off for different PINN-based methods (Similar to Figure 1) in case of reaction and wave equations is presented in Figure~\ref{supp:bar1}. 

\begin{figure}[htbp]
    \centering
    \includegraphics[width=0.75\linewidth]{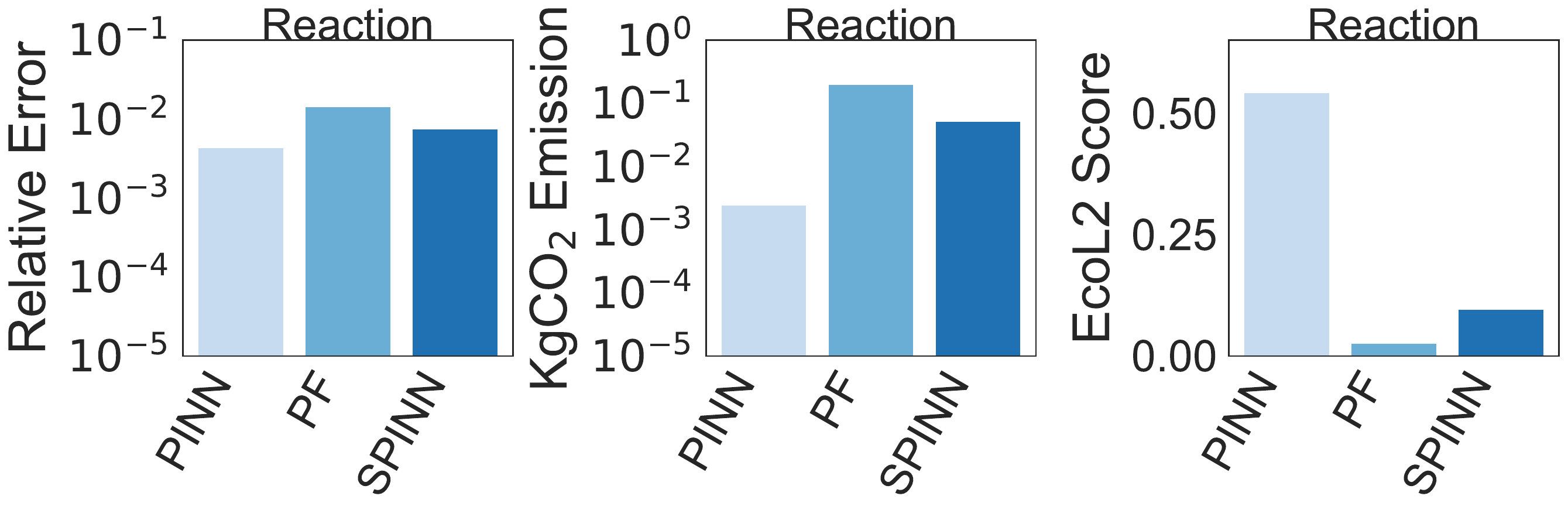}
    \includegraphics[width=0.75\linewidth]{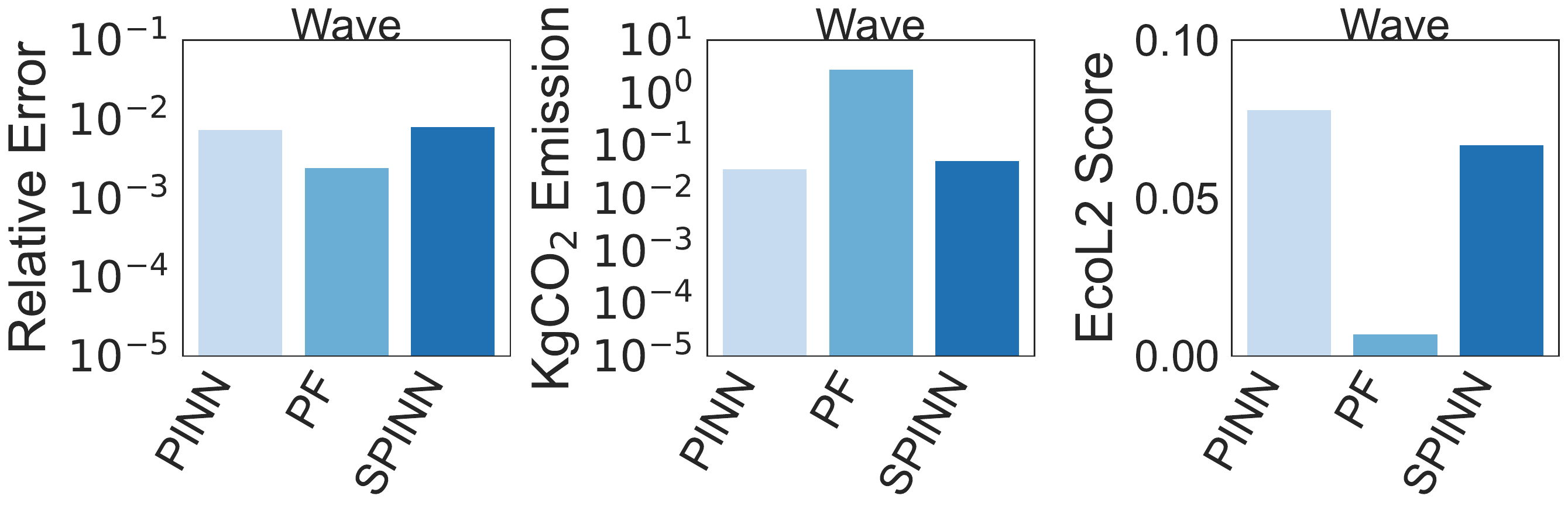}
    \caption{Rows show performance comparison of physics-informed learning methods on the reaction and wave equations, respectively. Models are evaluated on relative error \textbf{(Left)}, the corresponding carbon emissions (kgCO2) \textbf{(Middle)}, and the proposed EcoL2 metric \textbf{(Right)}. Comparison of models solely based on relative error provides a myopic view as they have varying carbon footprints. The proposed EcoL2 metric (higher values are preferable) captures this trade-off,
    offering a performant perspective of solver performance.}
    \label{supp:bar1}
\end{figure}

\subsection{PINN}
\label{dev_pinn}

\textbf{Advection: }Several architectural and hyperparameter possibilities are investigated for this problem. The network uses four hidden layers—each comprising four possibilities of neurons: 16, 32, 64, and 128. The network used the Tanh activation function. Xavier uniform initialization \cite{glorot2010understanding} is employed with a fixed seed for reproducibility. The optimizer is LBFGS, configured with a learning rate of 0.1, maximum evaluation steps of 50{,}000, and a stopping criterion based on a $1 \times 10^{-6}$ loss threshold. The training uses mini-batches of size 200. The loss function combines supervised and physics-based components, controlled by a residual loss weight hyperparameter ($\lambda_r$), tested with values 0.1, 0.5, and 1.0. The results of the developmental stage are presented in Table~\ref{sm:advection_pinn}.

\textbf{Reaction: }The model and training settings are tuned through specific hyperparameters. The network uses four hidden layers—each comprising four possibilities of neurons: 16, 32, 64, and 128. The network used the Tanh activation function. Xavier uniform initialization is employed with a fixed seed for reproducibility. $\lambda_r$ is taken to be 0.1, 0.5, and 1.0. Training is performed over 2,500 epochs using the LBFGS optimizer with a learning rate of 0.1, a maximum of 1 internal iteration per step, and 50,000 evaluations. The dataset consists of 200 points for each initial, left-boundary, and right-boundary condition and 10,000 residual (interior) points. The model is trained with batch size 200, and early stopping is enabled with a loss threshold of \(1 \times 10^{-6}\). The results of the developmental stage are presented in Table~\ref{sm:reaction_pinns}.

\textbf{Wave: }The PINN models are trained for wave equation. The network uses four hidden layers—each comprising four possibilities of neurons: 16, 32, 64, and 128. The network used the Tanh activation function. Xavier uniform initialization is employed with a fixed seed for reproducibility. $\lambda_r$ is taken to be 0.1, 0.5, and 1.0. Training is performed using the LBFGS optimizer with a learning rate of 0.1, a maximum of 1 iteration per step, and up to 50,000 function evaluations. The dataset includes 200 points each for the initial and boundary conditions and 10,000 collocation points for the residual. The model is trained for 10{,}000 epochs, and early stopping is optionally enabled based on a loss threshold of \(1 \times 10^{-6}\). The results of the developmental stage are presented in Table~\ref{sm:wave_pinns}. 

\subsection{PINNsFormer}
\label{dev_pf}

\textbf{Advection: }In this experiment, a PINNsFormer model is trained using a range of hyperparameter configurations to study their impact on performance and efficiency. The model architecture includes a transformer-based structure with a single encoder block and two attention heads, where the embedding dimension varies among \{16, 32, 64, 128\}. Each model is trained for 750 steps using the LBFGS optimizer with the strong Wolfe line search. The residual loss term is weighted by a hyperparameter $\lambda_r$, which is varied across \{0.1, 0.5, 1.0\}. All experiments are initialized with Xavier uniform weight initialization, and a fixed random seed is used for reproducibility. The results of the developmental stage are presented in Table~\ref{sm:advection_pf}. 

\textbf{Reaction: }This experiment evaluates the PINNsFormer model on a reaction-based PDE using a hyperparameter sweep over the number of neurons and residual loss weights. The transformer-based architecture consists of one encoder block with two attention heads, an output dimension of 1, and a hidden feedforward dimension 512. The embedding size is varied across \{16, 32, 64, 128\}, and the residual loss weight is tested with values \{0.1, 0.5, 1.0\}. All weights are initialized using Xavier uniform initialization, and a fixed seed is used for reproducibility. Training is performed over 750 steps using the LBFGS optimizer with a strong Wolfe line search. The results of the developmental stage are presented in Table~\ref{sm:reaction_pf}. 

\textbf{Wave: }PINNsFormer model on wave problem is applied using a transformer-inspired neural network architecture. The network comprises one encoder block with two attention heads, an embedding dimension of 32, and a hidden feedforward size of 2. Xavier uniform initialization is applied to all layers, and training is performed for 10 steps using the LBFGS optimizer with a strong Wolfe line search. The residual term is weighted by a hyperparameter $\lambda_r$, set to 0.1. The results of the developmental stage are presented in Table~\ref{sm:wave_pf}. 

\subsection{SPINN}
\label{dev_spinn}

\textbf{Advection: }This experiment SPINN uses a sweep over key hyperparameters to analyze performance and efficiency. The model architecture consists of a modified MLP with four hidden layers and output dimension 1. The number of neurons per layer is varied over \{16, 32, 64, 128\}, and no positional encoding is used. The residual loss term is weighted by a hyperparameter $\lambda_r$, tested over values \{0.1, 0.5, 1.0\}. Training is performed for 10,000 epochs using the Adam optimizer \cite{kingma2017adammethodstochasticoptimization} and learning rate of \(1 \times 10^{-3}\). The results of the developmental stage are presented in Table~\ref{sm:advection_spinn}. 

\textbf{Reaction: }The neural network is a modified MLP with four hidden layers and no positional encoding, and the number of neurons per layer (\texttt{features}) is varied across \{16, 32, 64, 128\}. The residual loss term is scaled by a hyperparameter $\lambda_r$, evaluated at \{0.1, 0.5, 1.0\}. Training is conducted over 10{,}000 epochs using the Adam optimizer with a learning rate of \(1 \times 10^{-3}\), and logging is performed every 5{,}000 steps. All runs are initialized with a fixed seed to ensure reproducibility. The results of the developmental stage are presented in Table~\ref{sm:reaction_spinn}. 

\textbf{Wave: }The experiment is configured as a modified MLP with four hidden layers, an output dimension of 1, and no positional encoding. The width of the network is varied across \{128, 64, 32, 16\}, and the residual weighting parameter $\lambda_r$ is swept across \{0.1, 0.5, 1.0\}. Training is performed using the Adam optimizer with a learning rate of \(1 \times 10^{-4}\) for 80,000 epochs. Emissions during training are tracked using \texttt{CodeCarbon}. The results of the developmental stage are presented in Table~\ref{sm:wave_spinn}. 

Figure~\ref{pie} (Left two subfigures) shows the pie chart distribution of different stages of carbon footprint for PINNs when solving the reaction and wave equations. It is evident that developmental stage entails a significant proportion when developmental and operational stages have similar number of iterations to train. For wave equation, where the developmental stage is constrained to fewer epochs because, in practice, many epochs are required to converge the loss function for this problem, the proportions for developmental and operational carbon are similar.

\begin{table}[H]
\centering
\caption{\textbf{Advection PINNs:} EcoL2 evaluation across residual weights and neurons.}
\begin{tabular}{cccccc}
\toprule
$\lambda_r$ & Neurons & EcoL2 & $\mathcal{R}$ & $C$ \\
\midrule
0.5 & 32  & 0.779 & 3.76e-4 & 5.24e-4 \\
0.1 & 32  & 0.763 & 5.42e-4 & 5.49e-4 \\
1.0 & 16  & 0.761 & 7.55e-4 & 3.82e-4 \\
0.5 & 16  & 0.750 & 9.72e-4 & 3.81e-4 \\
1.0 & 32  & 0.746 & 7.71e-4 & 5.84e-4 \\
0.5 & 64  & 0.745 & 3.33e-4 & 1.07e-3 \\
1.0 & 64  & 0.732 & 4.85e-4 & 1.06e-3 \\
0.1 & 16  & 0.727 & 1.49e-3 & 4.05e-4 \\
0.1 & 64  & 0.722 & 5.74e-4 & 1.10e-3 \\
1.0 & 128 & 0.638 & 6.50e-4 & 2.49e-3 \\
0.1 & 128 & 0.638 & 6.78e-4 & 2.46e-3 \\
0.5 & 128 & 0.638 & 6.46e-4 & 2.50e-3 \\
\bottomrule
\end{tabular}
\label{sm:advection_pinn}
\end{table}

\begin{table}[H]
\centering
\caption{\textbf{Advection PINNsFormer:} EcoL2 evaluation across residual weights and neurons.}
\begin{tabular}{ccccc}
\toprule
$\lambda_r$ & Neurons & EcoL2 & $\mathcal{R}$ & $C$ \\
\midrule
0.1 & 64  & 0.325 & 1.22e-3 & 1.36e-2 \\
0.5 & 64  & 0.296 & 1.56e-3 & 1.55e-2 \\
0.1 & 32  & 0.269 & 1.02e-3 & 1.88e-2 \\
0.5 & 128 & 0.263 & 1.37e-3 & 1.89e-2 \\
0.1 & 128 & 0.244 & 8.52e-4 & 2.22e-2 \\
0.1 & 16  & 0.197 & 3.02e-3 & 2.64e-2 \\
0.5 & 32  & 0.163 & 1.43e-3 & 3.65e-2 \\
0.5 & 16  & 0.160 & 3.84e-3 & 3.39e-2 \\
1.0 & 16  & 0.155 & 4.55e-3 & 3.46e-2 \\
1.0 & 128 & 0.148 & 5.44e-4 & 4.45e-2 \\
1.0 & 32  & 0.128 & 5.44e-3 & 4.28e-2 \\
1.0 & 64  & 0.0002 & 1.03e-3 & 1.41e-2 \\
\bottomrule
\end{tabular}
\label{sm:advection_pf}
\end{table}

\begin{table}[H]
\centering
\caption{\textbf{Advection SPINN:} EcoL2 evaluation across residual weights and neurons.}
\begin{tabular}{ccccc}
\toprule
$\lambda_r$ & Neurons & EcoL2 & $\mathcal{R}$ & $C$ \\
\midrule
0.1 & 16  & 0.664 & 3.04e-3 & 7.80e-4 \\
0.1 & 32  & 0.645 & 2.56e-3 & 1.27e-3 \\
0.5 & 16  & 0.592 & 9.13e-3 & 7.95e-4 \\
0.5 & 32  & 0.587 & 6.79e-3 & 1.28e-3 \\
0.5 & 64  & 0.514 & 9.05e-3 & 2.46e-3 \\
1.0 & 16  & 0.463 & 4.09e-2 & 8.00e-4 \\
0.1 & 128 & 0.454 & 1.51e-3 & 6.64e-3 \\
1.0 & 64  & 0.454 & 2.00e-2 & 2.60e-3 \\
0.5 & 128 & 0.429 & 3.19e-3 & 6.61e-3 \\
0.1 & 64  & 0.409 & 2.56e-3 & 7.77e-3 \\
1.0 & 128 & 0.398 & 6.77e-3 & 6.62e-3 \\
1.0 & 32  & 0.226 & 2.17e-2 & 1.50e-2 \\
\bottomrule
\end{tabular}
\label{sm:advection_spinn}
\end{table}

\begin{table}[H]
\centering
\caption{\textbf{Reaction PINNs:} EcoL2 evaluation across residual weights and neurons.}
\begin{tabular}{ccccc}
\toprule
$\lambda_r$ & Neurons & EcoL2 & $\mathcal{R}$ & $C$ \\
\midrule
0.5 & 16  & 0.681 & 4.81e-3 & 7.0e-5 \\
0.1 & 16  & 0.680 & 4.87e-3 & 8.0e-5 \\
1.0 & 16  & 0.663 & 6.25e-3 & 7.9e-5 \\
0.5 & 32  & 0.656 & 6.70e-3 & 1.05e-4 \\
0.1 & 32  & 0.643 & 7.93e-3 & 1.07e-4 \\
0.5 & 64  & 0.639 & 7.81e-3 & 2.01e-4 \\
0.1 & 64  & 0.607 & 1.17e-2 & 2.03e-4 \\
0.1 & 128 & 0.603 & 1.03e-2 & 4.37e-4 \\
1.0 & 64  & 0.591 & 1.41e-2 & 2.06e-4 \\
0.5 & 128 & 0.562 & 1.71e-2 & 4.32e-4 \\
1.0 & 32  & 0.551 & 2.36e-2 & 1.09e-4 \\
1.0 & 128 & 0.491 & 3.69e-2 & 4.22e-4 \\
\bottomrule
\end{tabular}
\label{sm:reaction_pinns}
\end{table}

\begin{table}[H]
\centering
\caption{\textbf{Reaction PINNsFormer:} EcoL2 evaluation across residual weights and neurons.}
\begin{tabular}{ccccc}
\toprule
$\lambda_r$ & Neurons & EcoL2 & $\mathcal{R}$ & $C$ \\
\midrule
0.1 & 16  & 0.304 & 1.67e-2 & 9.37e-3 \\
0.1 & 32  & 0.279 & 3.90e-2 & 8.15e-3 \\
0.1 & 64  & 0.252 & 3.51e-2 & 1.05e-2 \\
0.1 & 128 & 0.242 & 2.79e-2 & 1.24e-2 \\
0.5 & 32  & 0.204 & 2.88e-2 & 1.63e-2 \\
0.5 & 128 & 0.195 & 4.06e-2 & 1.57e-2 \\
0.5 & 64  & 0.191 & 2.75e-2 & 1.83e-2 \\
0.5 & 16  & 0.172 & 4.26e-2 & 1.88e-2 \\
1.0 & 128 & 0.0013 & 9.78e-1 & 2.64e-2 \\
1.0 & 64  & 0.0013 & 9.80e-1 & 2.39e-2 \\
1.0 & 16  & 0.0012 & 9.83e-1 & 2.12e-2 \\
1.0 & 32  & 0.0012 & 9.83e-1 & 2.19e-2 \\
\bottomrule
\end{tabular}
\label{sm:reaction_pf}
\end{table}

\begin{table}[H]
\centering
\caption{\textbf{Reaction SPINN:} EcoL2 evaluation across residual weights and neurons.}
\begin{tabular}{ccccc}
\toprule
$\lambda_r$ & Neurons & EcoL2 & $\mathcal{R}$ & $C$ \\
\midrule
0.5 & 16  & 0.398 & 2.48e-2 & 3.87e-3 \\
1.0 & 16  & 0.389 & 2.66e-2 & 4.03e-3 \\
0.1 & 32  & 0.364 & 3.23e-2 & 4.44e-3 \\
1.0 & 32  & 0.353 & 3.79e-2 & 4.40e-3 \\
0.5 & 32  & 0.347 & 3.95e-2 & 4.54e-3 \\
0.1 & 16  & 0.340 & 5.30e-2 & 3.88e-3 \\
0.1 & 64  & 0.336 & 3.87e-2 & 5.10e-3 \\
0.5 & 64  & 0.324 & 4.36e-2 & 5.24e-3 \\
1.0 & 64  & 0.309 & 5.44e-2 & 5.18e-3 \\
0.1 & 128 & 0.297 & 4.65e-2 & 6.36e-3 \\
0.5 & 128 & 0.295 & 5.01e-2 & 6.20e-3 \\
1.0 & 128 & 0.288 & 5.79e-2 & 6.04e-3 \\
\bottomrule
\end{tabular}
\label{sm:reaction_spinn}
\end{table}

\begin{table}[H]
\centering
\caption{\textbf{Wave PINNs:} EcoL2 evaluation across residual weights and neurons.}
\begin{tabular}{ccccc}
\toprule
$\lambda_r$ & Neurons & EcoL2 & $\mathcal{R}$ & $C$ \\
\midrule
0.1 & 32  & 0.452 & 6.24e-2 & 1.5e-5 \\
0.1 & 64  & 0.401 & 3.12e-2 & 3.20e-3 \\
1.0 & 64  & 0.332 & 1.55e-1 & 2.9e-5 \\
0.1 & 16  & 0.325 & 1.63e-1 & 8.0e-6 \\
0.1 & 128 & 0.322 & 1.65e-2 & 8.34e-3 \\
0.5 & 64  & 0.313 & 8.14e-2 & 3.43e-3 \\
0.5 & 32  & 0.308 & 1.33e-1 & 1.51e-3 \\
0.5 & 128 & 0.282 & 3.76e-2 & 8.06e-3 \\
1.0 & 128 & 0.239 & 7.75e-2 & 7.85e-3 \\
1.0 & 32  & 0.178 & 3.43e-1 & 1.65e-3 \\
1.0 & 16  & 0.172 & 3.80e-1 & 1.03e-3 \\
0.5 & 16  & 0.167 & 3.96e-1 & 9.06e-4 \\
\bottomrule
\end{tabular}
\label{sm:wave_pinns}
\end{table}

\begin{table}[H]
\centering
\caption{\textbf{Wave PINNsFormer:} EcoL2 evaluation across residual weights and neurons.}
\begin{tabular}{ccccc}
\toprule
$\lambda_r$ & Neurons & EcoL2 & $\mathcal{R}$ & $C$ \\
\midrule
1.0 & 16  & 0.0180 & 8.12e-2 & 2.24e-1 \\
0.5 & 16  & 0.0174 & 8.12e-2 & 2.32e-1 \\
0.1 & 128 & 0.0163 & 2.45e-2 & 3.30e-1 \\
0.1 & 32  & 0.0151 & 1.25e-1 & 2.30e-1 \\
0.5 & 32  & 0.0151 & 1.25e-1 & 2.30e-1 \\
1.0 & 32  & 0.0151 & 1.25e-1 & 2.31e-1 \\
0.5 & 128 & 0.0149 & 2.45e-2 & 3.61e-1 \\
1.0 & 128 & 0.0129 & 2.45e-2 & 4.20e-1 \\
1.0 & 64  & 0.0123 & 3.36e-1 & 1.62e-1 \\
0.1 & 64  & 0.0123 & 3.36e-1 & 1.62e-1 \\
0.5 & 64  & 0.0123 & 3.36e-1 & 1.62e-1 \\
0.1 & 16  & 0.0120 & 5.08e-1 & 1.04e-1 \\
\bottomrule
\end{tabular}
\label{sm:wave_pf}
\end{table}

\begin{table}[H]
\centering
\caption{\textbf{Wave SPINN:} EcoL2 evaluation across residual weights and neurons.}
\begin{tabular}{ccccc}
\toprule
$\lambda_r$ & Neurons & EcoL2 & $\mathcal{R}$ & $C$ \\
\midrule
0.1 & 16  & 0.538 & 1.83e-2 & 7.80e-4 \\
0.1 & 32  & 0.480 & 2.79e-2 & 1.27e-3 \\
0.1 & 128 & 0.384 & 9.12e-3 & 6.64e-3 \\
0.1 & 64  & 0.327 & 1.81e-2 & 7.77e-3 \\
0.5 & 16  & 0.281 & 1.89e-1 & 7.95e-4 \\
0.5 & 64  & 0.257 & 1.69e-1 & 2.46e-3 \\
0.5 & 128 & 0.255 & 7.86e-2 & 6.61e-3 \\
0.5 & 32  & 0.208 & 2.93e-1 & 1.28e-3 \\
1.0 & 16  & 0.156 & 4.29e-1 & 8.00e-4 \\
1.0 & 64  & 0.142 & 4.04e-1 & 2.60e-3 \\
1.0 & 128 & 0.084 & 4.98e-1 & 6.62e-3 \\
1.0 & 32  & 0.065 & 4.39e-1 & 1.50e-2 \\
\bottomrule
\end{tabular}
\label{sm:wave_spinn}
\end{table}

\subsection{Operational stage}
\label{operational}
Three neural network methods are implemented with distinct hyperparameter setups for the advection equation. The PINN model utilized 200 initial, 400 boundary, and 1000 residual points, with a residual parameter of 0.5, employing 64 neurons across four hidden layers trained for 10,000 epochs. The PINNsFormer model featured 128 neurons, a residual parameter of 1, an output dimension of 1, a hidden dimension of 512, one transformer block, and two attention heads, trained for 1,000 epochs, adhering to the original PINNsFormer paper. The SPINN model adopted a residual parameter 0.1, with 128 neurons, a learning rate 1e-3, and 50,000 epochs, using 128 points each for boundary, initial, and collocation.

Three neural network methods are implemented with distinct hyperparameter setups for the reaction equation, the PINN model is configured with 10,000 collocation points, 200 initial points, 400 boundary points, a residual parameter of 0.5, and consisted of 16 neurons across four hidden layers trained over 10,000 epochs. The PINNsFormer used 16 neurons, a residual parameter of 0.1, an output dimension of 1, a hidden dimension 512, one transformer block, and two attention heads trained for 1,000 epochs. The SPINN employed a residual parameter of 0.5, 16 neurons, a learning rate 1e-3, and is trained for 80,000 epochs, maintaining 128 boundary, initial, and collocation points.

Three neural network methods are implemented with distinct hyperparameter setups for the wave equation.  The PINN method applied 200 initial, 400 boundary, and 10,000 residual points, a residual parameter of 0.1, 128 neurons distributed across four hidden layers, and completed training in 80,000 epochs.  The PINNsFormer configuration included 128 neurons, a residual parameter 0.1, an output dimension of 1, a hidden dimension of 512, one transformer block, two attention heads, and training across 1,000 epochs. Lastly, SPINN utilized a residual parameter of 0.1, 128 neurons, a learning rate of 1e-4, and 80,000 epochs for the wave equation.

\begin{figure}
    \centering
    \vspace{-10pt}
    \includegraphics[width=0.95\textwidth]{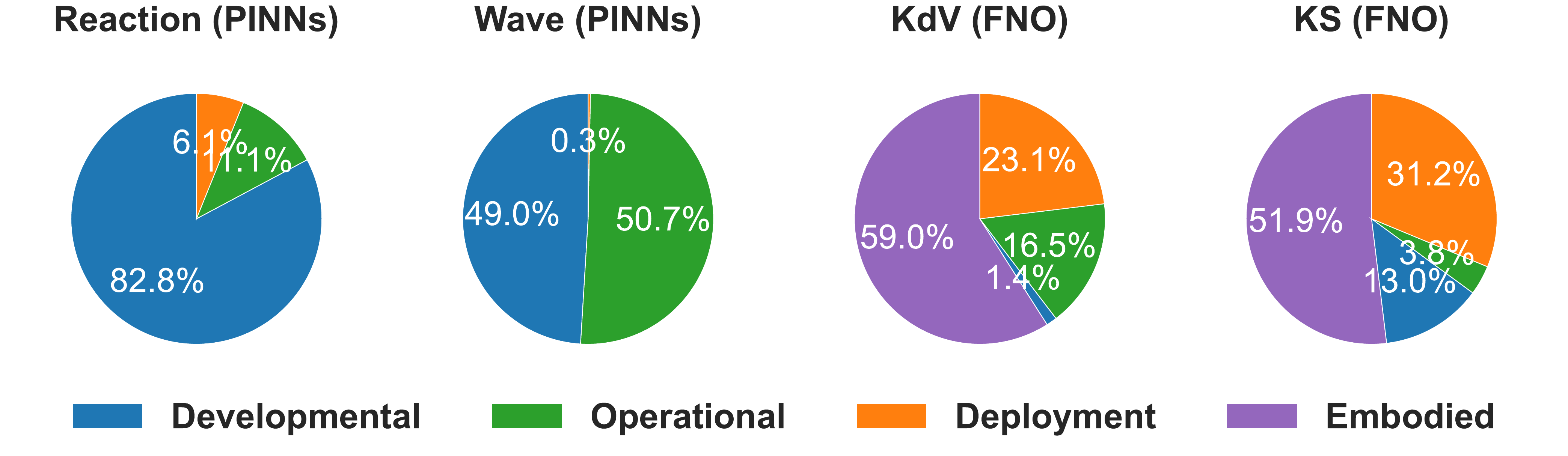}
    \caption{Component-wise Carbon Emissions across PDE Solvers (100 Inferences) }
    \label{pie}
    \vspace{-10pt}
\end{figure}

\section{Experimental details for neural operator methods}
\label{sm_sec:G}
This section presents the experimental details for neural operator methods, DON, FNO, and CNO for the KdV and KS equations. Subsection~\ref{sm:g1} presents the details about the dataset generation. Subsections~\ref{sm:g2},~\ref{sm:g3}, and~\ref{sm:g4} present the hyperparameter tuning and final run details of DON, FNO, and CNO, respectively. Figure~\ref{pie} (Right two subfigures) shows the pie chart distribution of different stages of carbon footprint for FNO when solving the KdV and KS equations. It is evident that the embodied carbon stage entails a significant proportion, which is the data generation through numerical solvers. Additionally, in contrast to PINN-based methods the carbon emission for deployment stage per 100 inferences is significantly more. The comparison trade-off for different neural operators (Similar to Figure 1) in the case of the KS equation is presented in Figure~\ref{supp:bar2}.

\subsection{Dataset generation}
\label{sm:g1}
Neural operator methods require input-output data for supervised learning. The objective for both the KdV and KS equations is to learn a mapping from the initial condition to time 10.0. For learning, the dataset for both KdV and KS equations is generated through the open source code of the pseudospectral method provided in \cite{brandstetter2022lie, githubBrandstetterjohannesLPSDA}. In particular, 2000 input-output function pairs are generated, where the initial condition is generated using the following parameterized series:

\begin{equation}
u(x, 0) = \sum_{i=1}^N A_i \sin\left(\frac{2\pi l_i x}{L} + \phi_i\right),
\end{equation}
where \( A_i \) are the amplitudes, \( \phi_i \) are the random phases drawn from the standard normal distribution, and \( l_i \in \mathbb{Z} \) are integer-valued frequencies uniformly sampled between 1 and 5, and \( L \) is the domain length (128 for KdV and 64 for KS). The base parameter set is perturbed to generate a family of \( 2000 \) diverse initial conditions. Small perturbations (\( \epsilon_A = 0.05 \), \( \epsilon_\phi = 0.25 \)) are applied to amplitudes and phases to encourage variability in initial waveforms.    

Once the initial conditions are generated, they are propagated forward in time, and the solution at the final time \( T = 10 \) is stored. The KdV and KS equations are integrated numerically using a pseudospectral method to compute spatial derivatives in Fourier space efficiently. The temporal integration is performed using the implicit Radau solver via \texttt{solve\_ivp}. The derivatives are computed spectrally using \texttt{scipy.fftpack.diff}. The domain is defined over \( x \in [0, L] \) with 100 spatial points for the KdV equation and 256 points for the KS equation, and \( t \in [0, 10] \) over 100 time steps. The solver tolerances are set to \( 10^{-6} \) for absolute and relative tolerances. The dataset is split into training and testing sets, with 1000 samples each. 

\subsection{DON}
\label{sm:g2}
\textbf{KdV: }To learn the mapping from input functions to output solutions for the KdV equation, the architecture of DON comprises a branch network, which takes the input function values sampled at 100 points, and a trunk network, which takes spatial locations as input. Both networks are fully connected neural networks with five hidden layers of 40 neurons each and an output dimension of 100 representing basis coefficients. The model is trained using the Adam optimizer with a learning rate of \( 1 \times 10^{-5} \) for a maximum of 50{,}000 iterations, with early stopping based on the test loss. Training data includes 1000 samples, and the loss function is defined as the mean squared error between predicted and actual outputs.

\textbf{KS: }The training dataset comprises 500 samples, each consisting of an input function and a corresponding output function evaluated on a spatial grid of 256 points. The input is passed through the branch network, while the spatial locations are passed through the trunk network. Both networks are fully connected feedforward neural networks with five hidden layers of 40 neurons each, and their outputs are projected onto a 100-dimensional latent basis. The final prediction is obtained by performing an inner product between the outputs of the branch and trunk networks. The model is trained using the Adam optimizer with a learning rate of $1\times10^{-5}$ and a mean squared error loss. Training proceeds for up to 50,000 iterations, with early stopping triggered if the test loss falls below $10^{-5}$.

\subsection{FNO}
\label{sm:g3}
\textbf{KdV: }FNO is employed to learn the mapping from input initial conditions to the solution. For training, 500 samples are used with a batch size of 100 and a learning rate of \( 1 \times 10^{-4} \). Using the Adam optimizer, one thousand epochs, with early stopping with loss, converge below \( 10^{-5} \). The FNO architecture consists of a lifting layer (a fully connected layer from input to Fourier layer), followed by four spectral convolution blocks. Each convolution uses the Fourier transform to apply learned multipliers to selected Fourier modes. Each spectral layer is paired with a standard pointwise 1D convolution, and the outputs are passed through ReLU activations. Finally, a projection network reduces the hidden features to the output dimension.

\textbf{KS: }In this experiment, FNO is applied to learn the solution operator of the KS equation, a nonlinear partial differential equation characterized by spatiotemporal chaos. The training dataset consists of 500 input-output function pairs, where the input and the output are sampled on a spatial grid of 256 points. The 1D input is augmented by concatenating the spatial coordinate to each sample to embed spatial positional information, forming a 2D input. The model architecture comprises four spectral convolution layers utilizing 32 Fourier modes and a hidden width of 64 channels, combined with pointwise linear transformations and ReLU activations. The FNO model is trained using the Adam optimizer with a $1\times10^{-4}$ learning rate and a mean squared error (MSE) loss function for up to 1000 epochs with batch size 100. An early stopping criterion is enforced if the test loss drops below $10^{-5}$. 

\subsection{CNO}
\label{sm:g4}
\textbf{KdV: }CNO learns a mapping from an input function to the corresponding output solution across a spatial grid. Only a subset of 500 training samples is used. Each input function is concatenated with the spatial coordinate to encode position information explicitly. The neural network architecture is based on a UNet-style encoder-decoder structure augmented with residual blocks. A custom activation function is used, which performs bicubic interpolation followed by a leaky ReLU activation. The architecture includes lifting layers to change the number of channels (to four) at input and output levels and residual block units (16 units) for feature refinement. Encoder feature maps are passed to the decoder using encoder-decoder expansion layers (4 layers) to maintain resolution alignment. A central bottleneck comprising 16 residual blocks is used. Training uses the Adam optimizer and mean squared error (MSE) as the loss criterion. The model is trained over 1000 epochs with an early stopping condition when the validation loss falls below $10^{-5}$. 

\textbf{KS: }The experiment involves learning the solution operator for the KS equation using CNO. The training dataset consists of 500 samples, each comprising an input function and the corresponding solution evaluated on a spatial grid of 256 points. The input function is augmented with corresponding spatial coordinates to encode positional information. The network consists of four downsampling and upsampling layers (encoder-decoder levels) and a central bottleneck comprising 16 residual blocks to capture complex nonlinear dynamics.  Encoder feature maps are passed to the decoder using encoder-decoder expansion layers (4 layers) to maintain resolution alignment. The model is trained using the Adam optimizer with a learning rate of $1\times10^{-4}$ and a mean squared error loss function. The training runs for up to 1000 epochs, triggering early stopping if the test loss falls below $10^{-5}$. Data is batched with a size of 100.

\begin{figure}[htbp]
    \centering
    \includegraphics[width=0.75\linewidth]{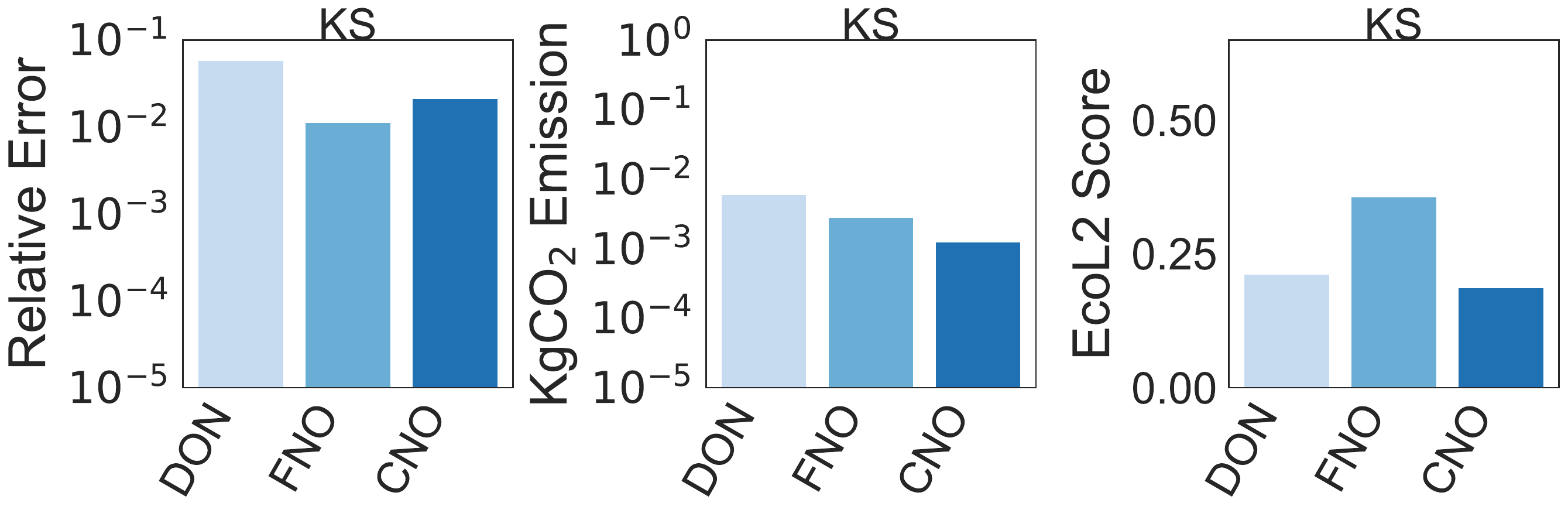}
    \caption{Performance comparison of operator learning methods on the KS equation. Models are evaluated on relative error \textbf{(Left)}, the corresponding carbon emissions (kgCO2) \textbf{(Middle)}, and the proposed EcoL2 metric \textbf{(Right)}. Comparison of models solely based on relative error provides a myopic view as they have varying carbon footprints. The proposed EcoL2 metric (higher values are preferable) captures this trade-off,
    offering a performant perspective of solver performance.}
    \label{supp:bar2}
\end{figure}

\section{Broader impact experiments}
\label{sm_sec:H}
This section presents the broader impacts of the proposed EcoL2 metric by showcasing its applicability on a higher dimensional neural PDE solver and two further tasks beyond neural PDE solving. The subsection~\ref{broader:hd} presents the application of SPINN on a higher-dimensional diffusion problem. The subsection~\ref{broader:func} presents that the EcoL2 metric is even applicable for function approximation tasks by showcasing its applicability on a benchmark experiment on Gaussian Process, Neural Processes, and Deep neural networks. Finally, subsection~\ref{broader:symbolic} presents that the EcoL2 metric can also be used in problems involving symbolic model discovery through an experiment on sparse regression-based model discovery.

\subsection{High-dimensional problem}
\label{broader:hd}

The time‐dependent diffusion equation in \(1+1\), \(2+1\), and \(3+1\) dimensions is considered
\[
\frac{\partial u}{\partial t} \;-\;\Delta u \;=\; 0,
\]
on the spatial domain \(X\in[-1,1]^n\) and time interval \(t\in[0,1]\), whose analytic solution is
\[
u(x,t) = \frac{\lVert X\rVert^2}{n} + 2t.
\]
The employed SPINN model is built following the modified MLP architecture \cite{cho2023separable} with 4 hidden layers of 64 neurons each and tensor rank 64.  The Adam optimizer with learning rate of \(10^{-3}\) is used with 32 collocation points per axis uniformly randomly sampled.  The number of epochs is 350/500/450 for 1+1D/2+1D/3+1D respectively. The residual‐regularization weight is set to 0.1. The collocation points are resampled every 100 epochs. As the method is PINN based, embodied carbon is zero. Also, no hyperparameter optimization is carried out and confirmation bias is used to select the hyperparameters and showcase the ratioanle of the experiment.

\begin{table}[H]
\vspace{-10pt}
\centering
\caption{Comparison of methods across various dimensionality using different metrics}
\label{tab:heatmultidimension}
\renewcommand{\arraystretch}{1.3}
\resizebox{\textwidth}{!}{%
\begin{tabular}{@{}c|c|c|c|c|c|c|c|c|c|c|c@{}|}
\cline{2-12}
\multicolumn{1}{c|}{} & 
\multirow{2}{*}{\diagbox{Methods}{Metrics}} & 
\multicolumn{4}{c|}{Traditional metrics} & 
\multicolumn{6}{c|}{Proposed carbon-aware metric} \\
\cline{3-12}
\multicolumn{1}{c|}{} & 
& $\mathcal{R}$ & RMSE & ME & MAE & $C_e$ & $C_d$ & $C_o$ & $C_i$ & $C$ & EcoL2 \\
\cline{2-12}

\multicolumn{1}{c|}{\rotatebox{90}{2D}} 
& SPINN         & 7.7e-3 & 1.1e-2 & 3.2e-2 & 9.5e-3 & - & - & 1.3e-4 & 3.6e-5 & 1.67e-4 & 0.64 \\
\cline{2-12}

\multicolumn{1}{c|}{\rotatebox{90}{3D}} 
& SPINN         & 7.3e-3 & 1.1e-2 & 6.2e-2 & 7.6e-3 & - & - & 2.7e-4 & 6.95e-7 & 2.7e-4 & 0.635 \\
\cline{2-12}

\multicolumn{1}{c|}{\rotatebox{90}{4D}} 
& SPINN         & 1.0e-2 & 1.5e-2 & 1.7e-1 & 1.1e-2 & - & - & 2.1e-3 & 8.01e-7 &2.1e-3 & 0.52 \\
\cline{2-12}

\end{tabular}
}
\vspace{-10pt}
\end{table}

Table~\ref{tab:heatmultidimension} compares the performance of the SPINN model across increasing problem dimensionality using both traditional metrics (such as $\mathcal{R}$ and RMSE) and our proposed carbon-aware metrics. While deep neural networks inherit the capacity to overcome the curse of dimensionality in approximation accuracy, the table reveals a less-discussed consequence: the steep increase in carbon emissions with higher dimensionality. For instance, while the 4D case, in traditional error metrics sometimes comparable to the lower-dimensional counterparts (RMSE), it incurs a carbon footprint (e.g., $C_o$ and total $C$) that is an order of magnitude higher. This growing discrepancy is captured succinctly by the EcoL2 score, which declines significantly as dimensionality increases. These results emphasize the importance of employing the EcoL2 metric for a more holistic assessment of the model performance.

\subsection{Function approximation}
\label{broader:func}

The target function defined as $f(x) = -\cos(2\pi x) + 0.5\sin(6\pi x)$ represents a nonlinear, periodic signal composed of both low and high-frequency components. This function serves as a benchmark for evaluating the learning capabilities of various machine learning models \cite{gpssLabs}. In this study, Gaussian Processes (GP), Neural Processes (NP) \cite{garnelo2018neural}, and Deep Neural Networks (DNN) are utilized to learn and approximate the function from sampled input-output pairs.

\textbf{Experimental details using Gaussian processes: }This experiment employs GP regression to approximate a nonlinear function using Bayesian optimization, starting with an initial set of 5 training samples at input points \([0.1, 0.3, 0.5, 0.7, 0.9]\). Throughout 10 Bayesian optimization iterations, one new sample is added per iteration based on the Expected Improvement (EI) acquisition function, leading to 15 training samples by the end. At each step, the EI criterion guided the selection of the following sampling point by balancing exploration and exploitation, using a small exploration parameter \(\xi = 0.01\) to discover new regions over purely exploiting known high-performing areas.

The GP model is defined using a composite kernel: a Constant Kernel initialized with a value of 1.0 (bounds: \([10^{-3}, 10^{3}]\)) multiplied by an RBF (Radial Basis Function) Kernel with an initial length scale of 0.2 (bounds: \([10^{-2}, 10^{2}]\)). The model is trained with \texttt{n\_restarts\_optimizer=10} to avoid local optima during kernel hyperparameter tuning and is set to normalize the target values for improved numerical stability. This enabled the GP to iteratively refine its approximation of the proper function iteratively, leveraging the acquisition function to sample informative points across the input space strategically. For testing and visualization, the model's predictions are evaluated over a dense grid of 1000 evenly spaced input points in the interval \([0, 1]\), allowing the comparison of predicted mean, uncertainty, and the proper function behavior as shown in Figure~\ref{sm:func}. Carbon emissions for both the training and testing phases are logged using \texttt{CodeCarbon} and are presented in Table~\ref{sup:tab11}.

\textbf{Experimental details using neural processes: }The experiment trains an NP model to learn a distribution over functions, targeting the function \( f(x) = -\cos(2\pi x) + 0.5\sin(6\pi x) \). During training, each batch consists of 16 function realizations, where each sample contains 10 context points and 50 target points randomly drawn from the interval \([0, 1]\), resulting in \(16 \times (10 + 50) = 960\) function evaluations per batch. The model architecture includes an encoder and decoder: the encoder maps concatenated \((x, y)\) pairs through two fully connected layers with ReLU activations to produce 128-dimensional representations; the decoder takes the aggregated context representation and target inputs and outputs the mean and standard deviation of the predicted target values using a three-layer feedforward network. The Adam optimizer is used with a learning rate is \(1 \times 10^{-3}\) for 50{,}000 steps, minimizing the negative log-likelihood loss. During testing, the trained NP model is evaluated on a fixed set of 10 evenly spaced context points and predicts over 1000 densely sampled target points in \([0, 1]\), with performance assessed by comparing predicted means and uncertainty to the actual function values, are shown in Figure~\ref{sm:func}. Carbon emissions for both the training and testing phases are logged using \texttt{CodeCarbon} and are presented in Table~\ref{sup:tab11}.

\textbf{Experimental details using deep neural networks: }This experiment evaluates the ability of DNN to approximate the nonlinear function \( f(x) = -\cos(2\pi x) + 0.5\sin(6\pi x) \) using supervised learning. A training dataset of 1000 input-output pairs is generated by sampling \( x \) randomly from the interval \([0, 1]\) and computing the corresponding \( f(x) \). The DNN consists of three fully connected layers, with two hidden layers of 64 neurons each and Tanh activation functions. The model is trained for 5000 epochs using the Adam optimizer with a learning rate of 0.01, minimizing the mean squared error (MSE) loss. After training, the model's performance is evaluated on a dense grid of 1000 evenly spaced test points from \([0, 1]\), and its prediction accuracy is quantified by computing the relative error, $\mathcal{R}$. Figure~\ref{sm:func} shows the prediction and accurate function. Emissions during the training and testing phases are tracked using \texttt{CodeCarbon} and presented in Table~\ref{sup:tab11}. 

\begin{figure}[!ht]
  \centering
  \includegraphics[width=0.32\textwidth]{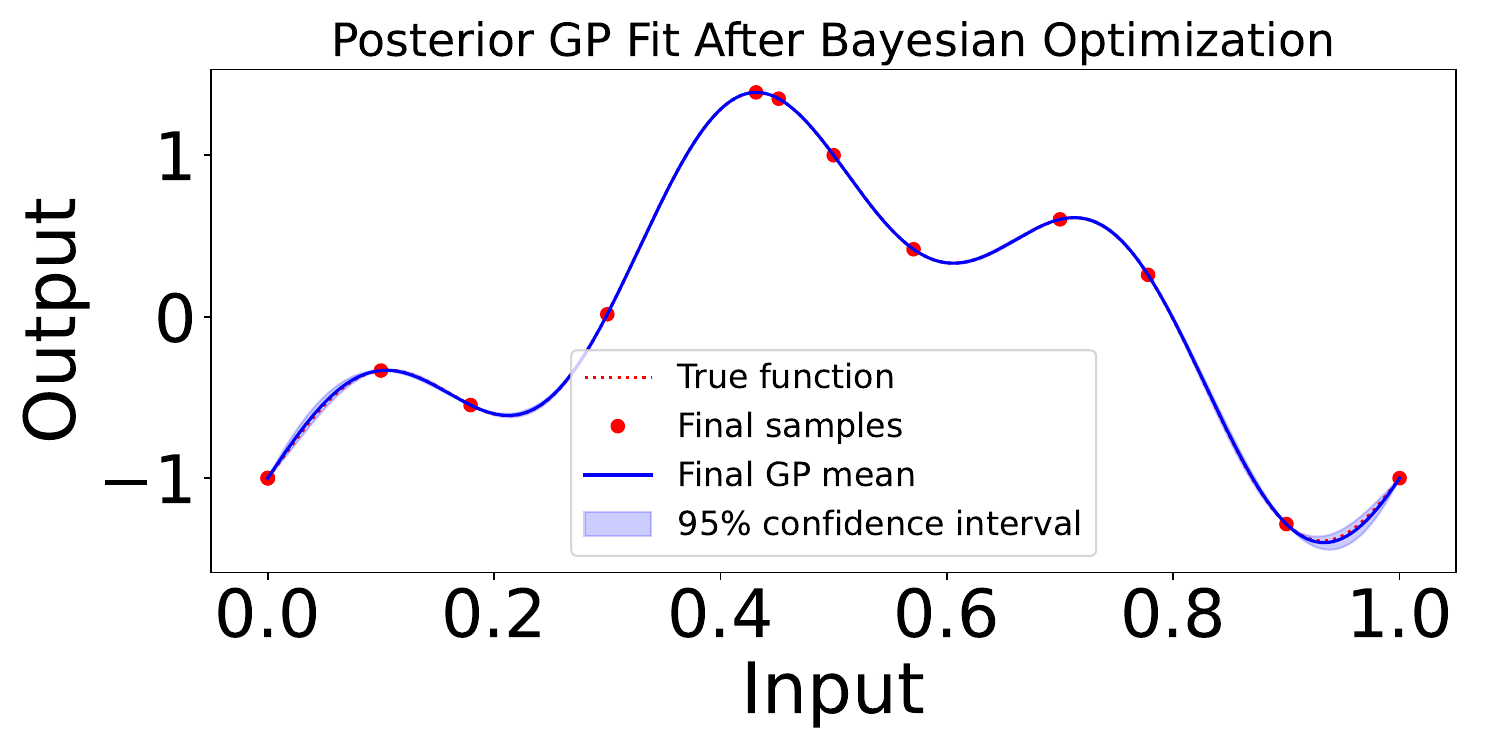} 
  \includegraphics[width=0.32\textwidth]{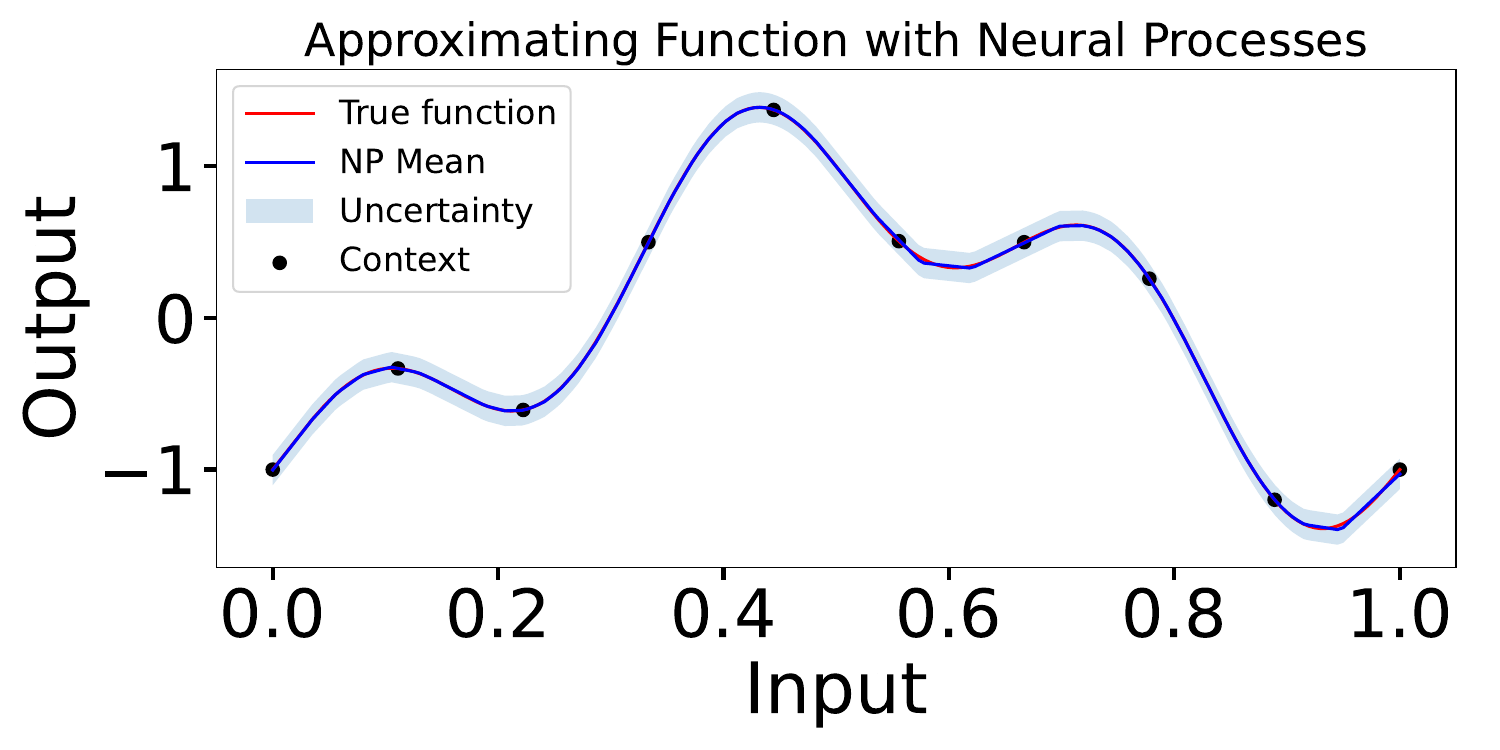} 
    \includegraphics[width=0.32\textwidth]{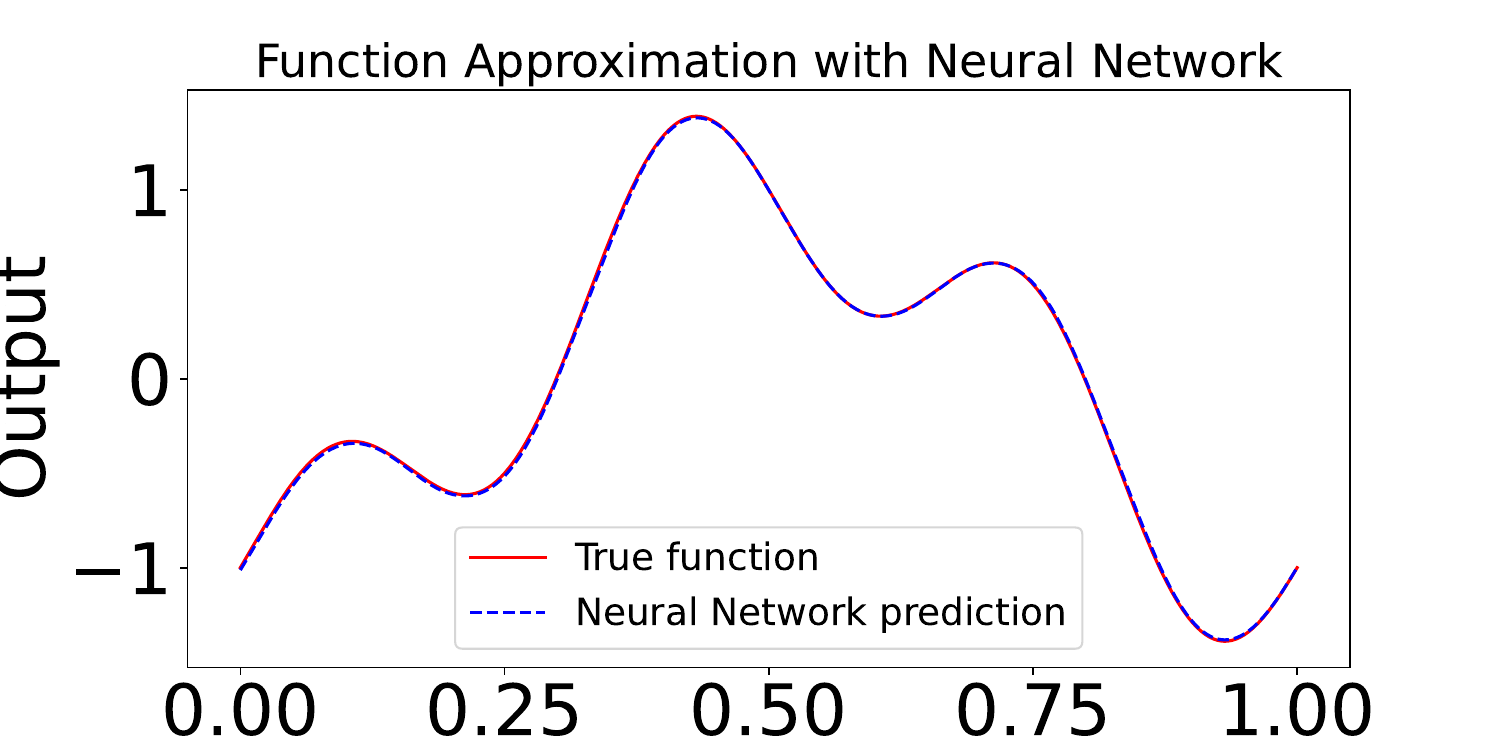}
  \caption{Function approximation performance. \textbf{Left:} GP; \textbf{Middle:} NP; \textbf{Right:} DNN.} 
  \label{sm:func}
\end{figure}

\begin{table}[!ht]
\centering
\caption{Function approximation task}
\label{sup:tab11}
\begin{tabular}{|c|c|c|c|c|c|}
\hline
Method & $C_0$ & $C_i$ & C & $R$ & EcoL2 \\
\hline
Gaussian Processes (GP)     & 9.75e-7    &  1.29e-7 & 1.104e-6  &  0.000041 & 0.8878    \\
Neural Processes  (NP)    &  3.31e-7 &  2.26e-9   & 3.33e-7  &   0.000054 &   0.8813  \\
Deep Neural Networks (DNN) &  1.58e-6  & 2.89e-9  & 1.58289  & 0.000167  & 0.8485     \\
\hline
\end{tabular}
\end{table}

\begin{figure}[!ht]
  \centering
  \includegraphics[width=0.60\textwidth]{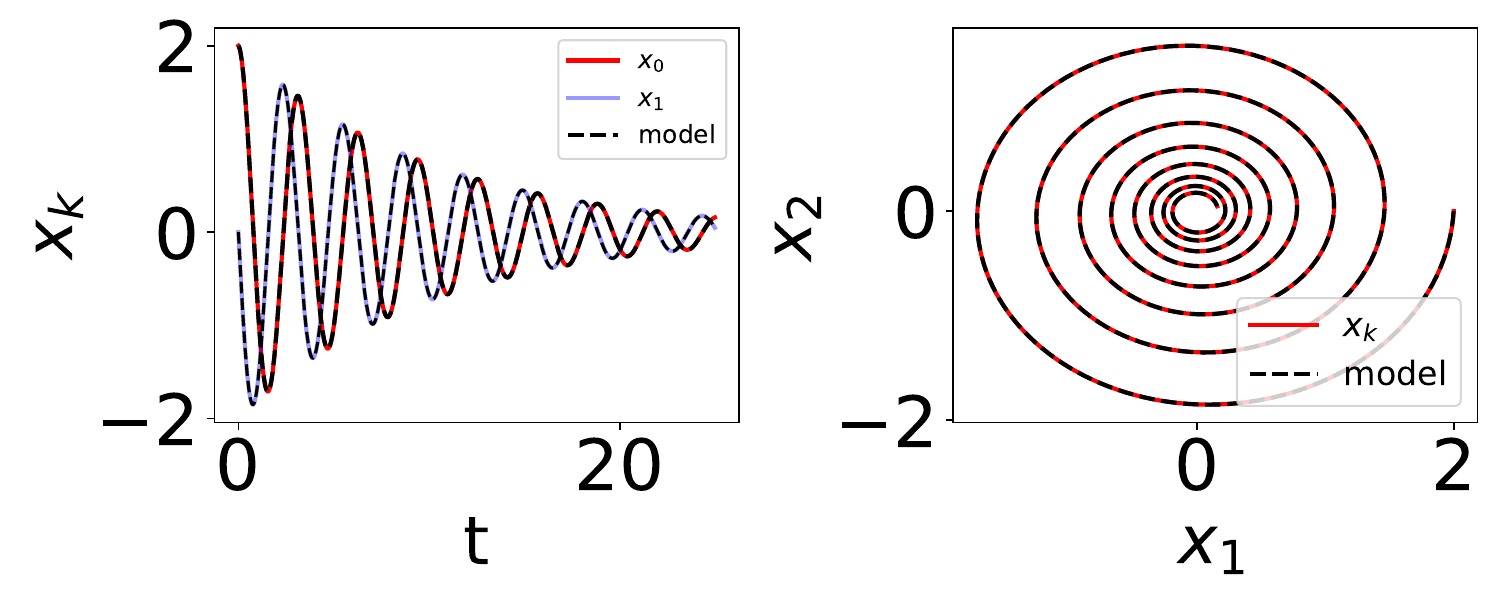}
  \caption{ODE Discovery}
  \label{sm:ode}
\end{figure}

\begin{table}[ht!]
\centering
\caption{Symbolic Regression}
\label{sup: Tab3}
\begin{tabular}{|c|c|c|c|c|}
\hline
 $C_0$ & $C_i$ & C & $R$ & EcoL2 \\
\hline
7.13e-7 & 1.81e-7  & 8.94e-7  & 0.000000779  & 0.9527  \\
\hline
\end{tabular}
\end{table}

\subsection{Symbolic regression}
\label{broader:symbolic}
This experiment aims to showcase the potential of EcoL2 on symbolic model discovery tasks. A benchmark example from \cite{brunton2016discovering} is used with the method therein, sparse identification of nonlinear dynamics (SINDy), to discover the linear damped simple harmonic oscillator given by,
\begin{align*}
\dot{x}_0 &= x_1 \\
\dot{x}_1 &= -x_0 - 0.1 x_1
\end{align*}
This form corresponds to the standard linear damped simple harmonic oscillation, where $x_0$ denotes the position, $x_1$ the velocity, and the damping coefficient is $0.1$. The experiment aims to discover the underlying ordinary differential equations (ODEs) using the \texttt{PySINDy} library \cite{de2020pysindy}. The goal is to recover the system’s governing equations from time-series data using sparse regression techniques. The experiment begins by simulating training data from a known linear damped SHO system using \texttt{solve\_ivp}, where the initial condition is $\mathbf{x}_0 = [2, 0]$, and the integration is carried out with a high-accuracy solver (LSODA) using very tight relative and absolute tolerances of $1 \times 10^{-12}$. The time step is $\Delta t = 0.01$, and the simulation runs from 0 to 25 seconds. The PySINDy model is configured with a polynomial feature library of degree 5 and the Sequentially Thresholded Least Squares optimizer with a sparsity threshold 0.05. This combination enables the identification of a parsimonious set of terms that best represent the dynamics. After fitting, the model is used to simulate the system from the same initial condition, and the results are compared visually and quantitatively against the actual dynamics as shown in Figure~\ref{sm:ode} and Table~\ref{sup: Tab3}. 

The applicability of EcoL2 on neural PDE solvers of varying dimensionality, function approximation tasks, and symbolic model discovery presents its broader applicability across scientific machine learning methods.  

\section{Details of country-wise carbon intensity}
\label{sm_sec:I}
This section provides the details of the experiment conducted to assess the geographical sensitivity of carbon emissions. Specifically, the experiment simulated the KS equation using FNO at six different geographical locations: New Zealand, South Africa, Switzerland, the United Arab Emirates, the United Kingdom, and the United States of America.

The experiment utilized the \texttt{OfflineEmissionsTracker} from \texttt{CodeCarbon}, which includes a \texttt{country\_iso\_code} parameter. This argument allows the tracker to fetch region-specific carbon intensity values based on the local electricity generation mix. By changing this parameter, one can simulate how much carbon would be emitted by running the same computational experiment in different countries.

\begin{table}[h!]
\centering
\caption{Carbon intensity values for each country used in the study.}
\label{tab:carbon_intensity}
\begin{tabular}{lc}
\toprule
Country & Carbon Intensity \\
\midrule
New Zealand & 112.76 \\
South Africa & 707.69 \\
Switzerland & 34.84 \\
United Arab Emirates & 561.14 \\
United Kingdom & 237.59 \\
United States & 369.47 \\
\bottomrule
\end{tabular}
\end{table}

Table~\ref{tab:carbon_intensity} summarizes the carbon intensity for each country considered in this paper and shows that it varies significantly across countries. These differences are primarily driven by the share of renewable versus non-renewable energy sources in the national energy mix. For instance, Switzerland and New Zealand exhibit low carbon intensities due to their reliance on hydro and other low-carbon sources. In contrast, South Africa and the UAE have high carbon intensities due to their dependence on fossil fuels, especially coal and gas. This demonstrates that the same AI experiment can have vastly different environmental footprints, even with similar computational times, depending on where it is run.


\end{document}